\documentclass{article}

\usepackage[utf8]{inputenc} 
\usepackage[T1]{fontenc}    
\usepackage{hyperref}       
\usepackage{url}            
\usepackage{booktabs}       
\usepackage{amsfonts}       
\usepackage{nicefrac}       
\usepackage{microtype}      
\usepackage{xcolor}         

\usepackage{authblk}

\usepackage{fullpage}
\usepackage{amsmath,amsfonts}
\usepackage{amsthm}
\usepackage{natbib}
\usepackage{hyperref}

\usepackage[ruled,vlined,linesnumbered]{algorithm2e}
\usepackage{xcolor}
\usepackage{enumitem}
\usepackage{todonotes}

\SetKwInput{KwInput}{Input}
\SetKwInput{KwReturn}{Return}

\newcommand{\rE}{{\mathbb E}}

\newcommand{\rR}{{\mathbb R}}

\newcommand{\KL}{{\mathrm{KL}}}

\newcommand{\trace}{{\mathrm{trace}}}

\newcommand{\mI}{{\mathrm I}}
\newcommand{\vz}{{\mathbf z}}

\newtheorem{lemma}{Lemma}

\newtheorem{theorem}{Theorem}
\newtheorem{proposition}{Proposition}

\newtheorem{definition}{Definition}
\newtheorem{corollary}{Corollary}

\newcommand*\lrb[1]{\left[#1\right]}
\newcommand*\lrbb[1]{\left\{#1\right\}}
\newcommand*\lrp[1]{\left(#1\right)}
\newcommand*\lrn[1]{\left\|#1\right\|}
\newcommand*\lrw[1]{\left\langle#1\right\rangle}

\newcommand*\Ep[2]{\mathbb{E}_{#1}\left[#2\right]}

\title{When is the Convergence Time of Langevin Algorithms Dimension Independent? A Composite Optimization Viewpoint}

\author[1]{Yoav Freund$^{\diamond}$
}

\author[1]{Yi-An Ma$^{\diamond}$
}

\author[2,3]{Tong Zhang$^{\diamond}$
}

\affil[1]{University of California, San Diego}

\affil[2]{Google Research}

\affil[3]{The Hong Kong University of Science and Technology}

\date{}

\begin{document}

\maketitle

\begin{abstract}
There has been a surge of works bridging MCMC sampling and optimization, with a specific focus on translating non-asymptotic convergence guarantees for optimization problems into the analysis of Langevin algorithms in MCMC sampling.
A conspicuous distinction between the convergence analysis of Langevin sampling and that of optimization is that all known convergence rates for Langevin algorithms depend on the dimensionality of the problem, whereas the convergence rates for optimization are dimension-free for convex problems. Whether a dimension independent convergence rate can be achieved by Langevin algorithm is thus a long-standing open problem.
This paper provides an affirmative answer to this problem for large classes of either Lipschitz or smooth convex problems with normal priors. By viewing Langevin algorithm as composite optimization, we develop a new analysis technique that leads to dimension independent convergence rates for such problems.
\end{abstract}

\section{Introduction}
Two of the major themes in machine learning are point prediction and uncertainty quantification.
Computationally, they manifest in two types of algorithms: optimization and Markov chain Monte Carlo (MCMC).
While both strategies have developed relatively separately for decades, there is a recent trend in relating both strands of research and translating nonasymptotic convergence guarantees in gradient based optimization methods to those in MCMC~\cite{Dalalyan_JRSSB,Moulines_ULA, Dalalyan_user_friendly,Xiang_underdamped,Xiang_overdamped,dwivedi2018log, Mangoubi1,Mangoubi2,Eberle_HMC,Variance_Reduction_theory,Vempala_ULA,Sampling_as_optimization}.
In particular, the Langevin sampling algorithm~\cite{Langevin_origin,MALA,durmus2017} has been shown to be a form of gradient descent on the space of probabilities~\cite{JKO,Sampling_as_optimization,Bernton_JKO,durmus2019analysis}.
Many convergence rates on Langevin algorithm have emerged thenceforward, based on different assumptions on the posterior distribution~\citep[e.g.,][to list a few]{Dalalyan_underdamped,MCMC_nonconvex,YinTat_Vempala,Xiang_Nonconvex,shen2019randomized,Yian_underdamped,MCMC_higher,Quanquan2019}.

A conspicuous distinction between the convergence analysis of Langevin sampling and that of gradient descent is that all known convergence rates for Langevin algorithms depend on the dimensionality of the problem, whereas the convergence rates for gradient descent are dimension-free for convex problems.
This prompts us to ask:
\begin{center}
{\it Can Langevin algorithm achieve dimension independent convergence rate under the usual convex assumptions?}
\end{center}

In order to answer this question formally, we make two assumptions on the negative log-likelihood function.
One is that the negative log-likelihood is convex.
Another is that the negative log-likelihood is either Lipschitz continuous or Lipschitz smooth.
We also employ a known and tractable prior distribution that is strongly log-concave---often times taken to be a normal distribution---to serve as a parallel to the $L_2$ regularizer in gradient descent.

Under such assumptions, we answer the above highlighted question in the affirmative.
In particular, we prove that a Langevin algorithm converges similarly as convex optimization for this class of problems.
In the analysis, we observe that the number of gradient queries required for the algorithm to converge does not depend on the dimensionality of the problem for either the Lipschitz continuous log-likelihood or the Lipschitz smooth log-likelihood equipped with a ridge separable structure.

To obtain this result, we first follow recent works (reference~\cite{durmus2019analysis} in particular) and formulate the posterior sampling problem as optimizing over the Kullback-Leibler (KL) divergence, which is composed of two terms: (regularized) entropy and cross entropy.
We then decompose the Langevin algorithm into two steps, each optimizing one part of the objective function.
With a strongly convex and tractable prior, we explicitly integrate the diffusion along the prior distribution, optimizing the regularized entropy; whereas gradient descent over the convex negative log-likelihood optimizes the cross entropy.
Via analyzing an intermediate quantity in this composite optimization procedure, we achieve a tight convergence bound that corresponds to the gradient descent's convergence for convex optimization on the Euclidean space.
This dimension independent convergence time for Lipschitz continuous log-likelihood and Lipschitz smooth log-likelihood endowed with a ridge separable structure carries over to the stochastic versions of the Langevin algorithm.

\section{Preliminaries}
\subsection{Two Problem Classes}
We consider sampling from a posterior distribution over parameter $w\in\rR^d$, given the data set $\vz$:
\[
p(w|\vz) \propto p(\vz|w) \pi(w)
\propto \exp\lrp{-U(w)},
\]
where the potential function $U$ decomposes into two parts: $U(w) = \beta^{-1} \lrp{f(w)+g(w)}$.

While the formulation is general, in the machine learning setting, $f(w)$ usually corresponds to the negative log-likelihood, and $g(w)$ corresponds to the negative log-prior. The parameter $\beta$ is the temperature, which often takes the value of $1/n$ in machine learning, where $n$ is the number of training data. The key motivation to consider this decomposition is that we assume that $g$ is ``simple'' so that an SDE involving $g$ can be solved to high precision. We will take advantage of this assumption in our algorithm design.


\paragraph{Assumption on function $g$}
\begin{enumerate}[label=\rm{A0}]
    \item We assume that function $g$ is $m$-strongly convex ($g(w)-\frac{m}{2}\lrn{w}^2$ is convex)~\footnote{We also say that the density proportional to $\exp\lrp{-\beta^{-1}g(w)}$ is $\beta^{-1}m$-strongly log-concave in this case.} and can be explicitly integrated. \label{assumption:g}
\end{enumerate}

\paragraph{Assumption on function $f$}
We assume that function $f$ is convex (Assumption~\ref{assumption:1}) and consider two cases regarding its regularity.
\begin{itemize}
    \item In the first case, we assume that function $f$ is $G$-Lipschitz continuous (Assumption~\ref{assumption:Lip}).
    \item In the second case, we assume that function $f$ is $L$-Lipschitz smooth (Assumption~\ref{assumption:2}). We then instantiate the result by endowing it with a ridge separable structure (Assumptions~\ref{assumption:ridge1} and~\ref{assumption:ridge2}).
\end{itemize}

The first case stems from Bayesian classification problems, where one has a simple strongly log-concave prior and a log-concave and log-Lipschitz likelihood that encodes the complexity of the data.
Examples include Bayesian neural networks for classification tasks~\cite{Neal_BNN}, Bayesian logistic regression~\cite{Bayes_Logit}, as well as other Bayesian classification problems~\cite{Bayes_SVM} with Gaussian or Bayesian elastic net priors.
The second case corresponds to the regression type
problems, where the entire posterior is strongly log-concave and log-Lipschitz smooth.
In this case, one can separate the negative log-posterior into two parts: $\beta^{-1} g(w) = \frac{\beta^{-1}m}{2} \lrn{w}^2$ and $\beta^{-1} f(w) = \lrp{ -\log p(w|\vz) - \frac{\beta^{-1}m}{2} \lrn{w}^2 }$, which is convex and $\beta^{-1} L$-Lipschitz smooth.
We therefore directly let $g(w) = \frac{m}{2} \lrn{w}^2$ in Section~\ref{sec:smooth}.


\subsection{Objective Functional and Convergence Criteria}
We take the KL divergence $\beta^{-1} Q(p)$ to be our objective functional and solve the following optimization problem:
\begin{align}
p_* = & \arg\min_{p} Q(p), \label{eq:opt_ensemble}\\
&Q\lrp{p} = \int p(w) \ln\frac{p}{p(w|\vz)} d w = \rE_{w\sim p}\lrb{f(w) + g(w) + \beta\ln p(w)}.\nonumber
\end{align}
The minimizer that solves the optimization problem~\eqref{eq:opt_ensemble} is the posterior distribution:
\begin{align}
p_*(w)\propto\exp\lrp{-\beta^{-1} (f(w) + g(w))}.
\label{eq:expo}
\end{align}

We further define the entropy functional as
\[
H(p) = \beta\; \rE_{w \sim p} \ln p(w),
\]
so that the objective functional decomposes into the regularized entropy plus cross entropy:
\[
Q(p) = \big( H(p) + \rE_{w\sim p}\lrb{g(w)} \big) + \rE_{w\sim p}\lrb{f(w)}.
\]
With this definition of the objective function, we state that the difference in $Q$ leads to the KL divergence.
\begin{proposition}
Let $p$ be the solution of \eqref{eq:opt_ensemble}, and $p'$ be another distribution on $w$. We have
\[
\KL(p'\|p) = \beta^{-1} [ Q(p')-Q(p)] .
\]
\label{prop:KL}
\end{proposition}
This result establishes that the convergence in the objective $\beta^{-1}Q(p')$ is equivalent to the convergence in $\KL$-divergence. Therefore our analysis will focus on the convergence of $\beta^{-1} Q(p')$.

We also define the $2$-Wasserstein distance between two distributions that will become useful in our analysis.
\begin{definition}
Given two probability distributions $p(x)$ and $p'(y)$ on $\rR^d$, and
let $\Pi(p,p')$ be the class of distributions $q(x,y)$ on $\rR^d \times \rR^d$ so that the marginals
$q(x)=p(x)$ and $q(y)=p'(y)$.
The $W_2$ Wasserstein distance of $p$ and $p'$ is defined as
\[
 W_2(p,p')^2 = \min_{q \in \Pi(p,p')}
 \rE_{(x,y) \sim q} \|x-y\|_2^2  .
\]
\end{definition}
A celebrated relationship between the $\KL$-divergence and the $2$-Wasserstein distance is known as the Talagrand inequality~\cite{Villani_Talagrand}.
\begin{proposition}
Assume that probability density $p_*$ is $\widehat{m}$-strongly log-concave, and $p'$ defines another distribution on $\rR^d$.
Then $p_*$ satisfies the log-Sobolev inequality with constant $\widehat{m}/2$~\footnote{This fact follows from the Bakry-Emery criterion~\cite{bakryEmery}.}, and yields the following Talagrand inequality:
\[
W_2^2(p_*,p') \leq \widehat{m}^{-1} \KL(p_*\|p') .
\]
\label{prop:talagrand}
\end{proposition}

\section{Related Works}
Some previous works have aimed to sample from posteriors of the similar kind and obtain convergence in the KL divergence or the squared $2$-Wasserstein distance.
\paragraph{In the Lipschitz continuous case,}
where the negative log-likelihood is convex and $\widehat{G}$-Lipschitz continuous, composed with an $\widehat{m}$-strongly convex and $M$-Lipschitz smooth negative log-prior,
the convergence time to achieve $W_2^2(\widetilde{p}_T, p_*) \leq \epsilon$ is
$\widetilde{\Omega}\lrp{ \frac{d M + \widehat{G}^2}{\widehat{m}^2\epsilon^2} }$ \citep[Corollary 22 of][]{durmus2019analysis}.
Similarly, \cite{Niladri_nonsmooth} uses Gaussian smoothing to obtain a convergence time of $\widetilde{\Omega}\lrp{ \frac{d (M+\widehat{G}^2)}{\widehat{m}^2 \epsilon}}$ (in Theorem 3.4), which improves the dependence on accuracy $\epsilon$.
In~\cite{wenlong}, the Metropolis-adjusted Langevin algorithm is levaraged with a proximal sampling oracle to remove the polynomial dependence on the accuracy $\epsilon$ (in total variation distance) and achieve a $\widetilde{\Omega}\lrp{d\log(\frac{1}{\epsilon})}$ convergence time for a related composite posterior distribution.
Unfortunately, an additional dimension dependent factor is always introduced into the overall convergence rate.
This work demonstrates that if the $m$-strongly convex regularizer is explicitly integrable, then the convergence time for the Langevin algorithm to achieve $\mathrm{KL}(\widetilde{p}_T\|p_*)\leq\epsilon$ is dimension independent: $T = \Omega\lrp{\frac{\widehat{G}^2}{\widehat{m}\epsilon}}$. This is proven in Theorem~\ref{thm:1} for the full gradient Langevin algorithm, and in Theorem~\ref{thm:2} for the stochastic gradient Langevin algorithm.
Using Proposition~\ref{prop:talagrand}, the result implies a bound of
$T = \Omega\lrp{\frac{\widehat{G}^2}{\widehat{m}^2\epsilon}}$ to achieve $W_2^2(\widetilde{p}_T, p_*) \leq \epsilon$.

\paragraph{In the Lipschitz smooth case,}
where the negative log-posterior $U$ is $\widehat{m}$-strongly convex and $\widehat{L}$-Lipschitz smooth,
the overdamped Langevin algorithm has been shown to converge in $\widetilde{\Omega}\lrp{\frac{\widehat{L}^2}{\widehat{m}^2}\frac{d}{\epsilon}}$ number of gradient queries~\cite{Dalalyan_JRSSB,Dalalyan_user_friendly,Xiang_overdamped,durmus2017,Moulines_ULA,durmus2019analysis}, while the underdamped Langevin algorithm converges in $\widetilde{\Omega}\lrp{ \frac{\widehat{L}^{3/2}}{\widehat{m}^2} \sqrt{\frac{d}{\epsilon}} }$ gradient queries~\cite{Xiang_underdamped,Yian_underdamped,Dalalyan_underdamped},
to ensure that $\KL(\widetilde{p}_T\|p_*)\leq\epsilon$ and $W_2^2(\widetilde{p}_T,p_*)\leq\epsilon$.
Using a randomized midpoint integration method for the underdamped Langevin dynamics, this convergence time can be reduced to $\widetilde{\Omega}\lrp{ \frac{\widehat{L}}{\widehat{m}^{4/3}} \lrp{ \frac{d}{\epsilon} }^{1/3} }$ for convergence in squared $2$-Wasserstein distance~\cite{shen2019randomized}.
This paper establishes that for overdamped Langevin algorithm, the convergence time can be sharpened to $\Omega\lrp{\frac{\trace(\widehat{H}^2)}{\widehat{m}^2\epsilon}}$ to achieve $\KL(\widetilde{p}_T\|p_*)\leq\epsilon$, where matrix $\widehat{H}$ is an upper bound for the Hessian of function $U$.

Previous works have also focused on the ridge separable potential functions studied in this work.
There is a literature that requires incoherence conditions on the data vectors and/or high-order smoothness conditions on the component functions to achieve a $\widetilde{\Omega}\lrp{ \lrp{\frac{d}{\epsilon}}^{1/4} }$ convergence time for $W_2^2(\widetilde{p}_T,p_*)\leq\epsilon$ using Hamiltonian Monte Carlo methods~\cite{Mangoubi1,Mangoubi2}.
Making further assumptions that the differential equation of the Hamiltonian dynamics is close to the span of a small number of basis functions, this bound can be improved to polynomial in $\log(d)$~\cite{YinTat_Vempala}.
Another thread of work alleviates these assumptions and achieves the $\widetilde{\Omega}\lrp{\lrp{ \frac{d}{\epsilon} }^{1/4}}$ convergence time for the general ridge separable potential functions via higher order Langevin dynamics and integration schemes~\cite{MCMC_higher}.
We follow this general ridge separable setting and assume that each individual log-likelihood is Lipschitz smooth. Under this assumption, we demonstrate in this paper, by instantiating the bound for the general Lipschitz smooth case, that the Langevin algorithm converges in $\Omega\lrp{ \frac{1}{\epsilon} }$ number of gradient queries to achieve $\KL(\widetilde{p}_T\|p_*)\leq\epsilon$ (see Corollary~\ref{corollary:GD} and Corollary~\ref{corollary:sgld_smooth}).

\section{Langevin Algorithms}
We consider the following variant of the Langevin Algorithm~\ref{alg:langevin}.

\begin{algorithm}[H]
\label{alg:langevin}
\SetAlgoLined
\KwInput{Initial distribution $p_0$ on $\rR^d$, stepsize $\eta_t$, $\beta=1$}
Draw $w_0$ from $p_0$ \\
\For{$t = 1, 2, \ldots,T$}{
Sample $\widetilde{w}_t$ from $\widetilde{w}_t(\eta_t)$ with the following SDE on $\rR^d$ and initial value $\widetilde{w}_t(0)=w_{t-1}$
\begin{align}
    \widetilde{w}_t(\eta_t)
     = w_{t-1} - \int_{0}^{\eta_t} \nabla g(\widetilde{w}_t(s)) d s + \sqrt{2 \beta} \int_{0}^{\eta_t} d B_s , \label{eq:diffuse}
    \end{align}
    where $d B_s$ is the standard Brownian motion on $\rR^d$.

    Let
    \begin{equation}
    {w}_t = \widetilde{w}_t- \widetilde{\eta}_t \nabla f(\widetilde{w}_t) \label{eq:ld}
    \end{equation}
    }
    \Return{$\widetilde{w}_T$}
    \caption{Langevin Algorithm with Prior Diffusion}
\end{algorithm}

In this method, we assume that the prior diffusion equation \eqref{eq:diffuse} can be solved efficiently.
When the prior distribution is a standard normal distribution where
$g(w) = \frac{m}{2}\lrn{w}_2^2$ on $\rR^d$,
we can instantiate equation~\eqref{eq:diffuse} to be:
\begin{align}
\text{Sample}\quad    \widetilde{w}_t \sim \mathcal{N}\lrp{ e^{-m \eta_t}w_{t-1}, \frac{1-e^{-2m \eta_t}}{m} \beta \mI}.
    \label{eq:Langevin_sample}
\end{align}
%
%
In general, the diffusion equation \eqref{eq:diffuse} can also be solved numerically for separable $g(w)$ of the form
\[
g(w)=\sum_{j=1}^d g_j(w_j),
\]
where $w=[w_1,\ldots,w_d]$. In this case, we only need to solve $d$ one-dimensional problems, which are relatively simple.  For example, this includes the $L_1-L_2$ regularization arising from the Bayesian elastic net~\cite{Bayes_elastic_net},
\[
g(w)= \frac{m}{2} \|w\|_2^2 + \alpha \|w\|_1 ,
\]
among other priors that decompose coordinate-wise.

We will also consider the stochastic version of Algorithm~\ref{alg:langevin}, the stochastic gradient Langevin dynamics (SGLD) method, with a strongly convex function $g(w)$.
Assume that function $f$ decomposes into $f(w) = \frac{1}{n} \sum_{i=1}^n \ell(w,z_i)$.
Let $D$ be the distribution over the dataset $\Omega$ such that expectation over it provides the unbiased estimate of the full gradient: $\rE_{z\sim D}\nabla_w \ell(w,z) = \nabla f(w)$.
Then the new algorithm takes the following form and can be instantiated in the same way as Algorithm~\ref{alg:langevin}.

\begin{algorithm}[H]
\label{alg:sgld}
\SetAlgoLined
\KwInput{Initial distribution $p_0$ on $\rR^d$, stepsize $\eta_t$, $\beta=1/n$}
Draw $w_0$ from $p_0$ \\
\For{$t = 1, 2, \ldots,T$}{
Sample $\widetilde{w}_t$ from $\widetilde{w}_t(\eta_t)$ with the following SDE on $\rR^d$ and initial value $\widetilde{w}_t(0)=w_{t-1}$
\begin{align}
    \widetilde{w}_t(\eta_t)
     = w_{t-1} - \int_{0}^{\eta_t} \nabla g(\widetilde{w}_t(s)) d s + \sqrt{2 \beta} \int_{0}^{\eta_t} d B_s , \label{eq:sg_diffuse}
    \end{align}
    where $d B_s$ is the standard Bronwian motion on $\rR^d$.

    Draw minibatch $\mathcal{S}$ where each $z_i\in\mathcal{S}$ are i.i.d. draws: $z_i\sim D$. Let
    \begin{equation}
    {w}_t = \widetilde{w}_t- \widetilde{\eta}_t \frac{1}{|\mathcal{S}|} \sum_{z_i\in\mathcal{S}} \nabla_w \ell(\widetilde{w}_t,z_i). \label{eq:sgld-2}
    \end{equation}
    }
    \Return{$\widetilde{w}_T$}
    \caption{Stochastic Gradient Langevin Algorithm with Prior Diffusion}
\end{algorithm}
This algorithm becomes the streaming SGLD method where in each iteration we take one data point $z\sim D$.

In the analysis of Algorithm~\ref{alg:langevin}, we will use $p_{t-1}$ to denote the distribution of $w_{t-1}$, and $\widetilde{p}_t$ to denote the distribution of $\widetilde{w}_t$,
where the randomness include all random sampling in the algorithm.
We also denote $\mu_t$ and $\widetilde{\mu}_t$ as the measures associated with random variables $w_t$ and $\widetilde{w}_t$, respectively.
When using samples along the Markov chain to estimate expectations over function $\phi(\cdot)$, we take a weighted average, so that
\[
\hat{\phi}(p) = \frac{1}{\sum_{s=1}^T \eta_s} \sum_{t=1}^T  \eta_t \phi(\widetilde{w}_t),
\]
which is equivalent to the expectation with respect to the weighted averaged distribution:
\[
\bar{\mu}_T = \frac{1}{\sum_{s=1}^T \eta_s} \sum_{t=1}^T  \eta_t \widetilde{\mu}_t .
\]

We prove in what follows the convergence of the distribution $\widetilde{p}_t$ along the updates of \eqref{eq:diffuse} and \eqref{eq:ld} towards the posterior distribution \eqref{eq:expo}.

\section{Langevin Algorithms in Lipschitz Convex Case}
For the posterior $p(w|\vz) \propto \lrp{-\beta^{-1}(f(w)+g(w))}$, we assume that function $f$ satisfies the following two conditions common to convex analysis.
\paragraph{Assumptions for the Lipschitz Convex Case:}
\begin{enumerate}[label=\rm{A}{{\arabic*}}]
    \item Function $f:\rR^d\rightarrow\rR$ is convex. \label{assumption:1_Lip}
\end{enumerate}
\begin{enumerate}[label=\rm{$\rm{A2_L}$}, leftmargin=2.0\parindent]
    \item Function $f$ is $G$-Lipschitz continuous on $\rR^d$: $\|\nabla f(w)\|_2 \leq G$. \label{assumption:Lip}
\end{enumerate}
We also assume that function $g:\rR^d\rightarrow\rR$ is $m$-strongly convex.
Note that we have assumed that the gradient of function $f$ exists but have not assumed that function $f$ is smooth.

\subsection{Full Gradient Langevin Algorithm Convergence in Lipschitz Convex Case}
\label{sec:ld_cvg_Lip}

Our main result for Full Gradient Langevin Algorithm in the case that $f$ is Lipshitz can be stated as follows. 
\begin{theorem}
\label{thm:1}
Assume that function $f$ satisfies the convex and Lipschitz continuous Assumptions~\ref{assumption:1_Lip} and~\ref{assumption:Lip}. 
Further assume that function $g(w)$ satisfies Assumption~\ref{assumption:g}.
Then for $\widetilde{p}_T$ following the Langevin Algorithm~\ref{alg:langevin}, it satisfies (
for $\widetilde{\eta}_t = \lrp{1-e^{-m\eta_t}}/m = 2/\lrp{ m (t+2) }$):
\begin{equation*}
\sum_{t=1}^T
\frac{1+0.5 t}{T+0.25 T(T+1)}
\beta^{-1} [Q(\widetilde{p}_t) - Q(p_*) ]
\leq \frac{5 G^2}{\beta m T} .
\end{equation*}
By the convexity of the KL divergence, this leads to the convergence time of
\[
T = \frac{5 G^2}{\beta m \epsilon},
\]
for the averaged distribution to convergence to $\epsilon$ accuracy in the KL-divergence $\beta^{-1} Q$.
\end{theorem}
We devote the rest of this section to prove Theorem~\ref{thm:1}.
\begin{proof}[Proof of Theorem~\ref{thm:1}]
We take a composite optimization approach and analyze the convergence of the Langevin algorithm in two steps.
First we characterize the decrease of the regularized entropy $\rE_{w \sim p} \lrb{ g(w) + H(p) }$ along the diffusion step~\eqref{eq:diffuse}.
\begin{lemma}[For Regularized Entropy]
We generalize Lemma 5 of~\cite{durmus2019analysis} and have for $\widetilde{p}_t$ being the density of $\widetilde{w}_t$ following equation~\eqref{eq:diffuse} and $p$ being another probability density,
\[
\frac{2}{m} (1-e^{-m \eta_t}) \left( \rE_{w \sim \widetilde{p}_t} [g(w) + H(\widetilde{p}_t)] - \rE_{w \sim p} [g(w) + H(p)]\right) \leq  e^{-m \eta_t} W_{2}^2({p}_{t-1},p) - W_2^2(\widetilde{p}_{t},p) ,
\]
where $m$ is the strong convexity of $g(w)$.
\label{lem:entropy}
\end{lemma}
We then capture the decrease of the cross entropy $\rE_{w \sim p} \lrb{ f(w) }$ along the gradient descent step~\eqref{eq:ld}.
This result follows and parallels the standard convergence analysis of gradient descent~\citep[see][for example]{zinkevich2003online,zhang2004solving}. 
\begin{lemma}
Given probability density $p$ on $\rR^d$.
Define
\[
 f(p)= \rE_{w \sim p} f(w),
\]
then we have for $p_t$ being the density of $w_t$ following equation~\eqref{eq:ld}:
\[
2 \widetilde{\eta}_t [f(\widetilde{p}_t)- f(p)]
\leq  W_2^2(\widetilde{p}_t,p) - W_2^2(p_t,p)
+ \widetilde{\eta}_t^2 G^2 . 
\]
\label{lem:sgd0}
\end{lemma}
We then combine the two steps to prove the overall convergence rate for the Langevin algorithm.
It is worth noting that by aligning the diffusion step~\eqref{eq:diffuse} and the gradient descent step~\eqref{eq:ld} at $\widetilde{p}_t$, we combine $\rE_{w \sim \widetilde{p}_t} [g(w) + H(\widetilde{p}_t)]$ with $f(\widetilde{p}_t)$ and cancel out $W_2^2(\widetilde{p}_t,p)$ perfectly and achieve the same convergence rate as that of stochastic gradient descent in optimization.
\begin{proposition}
Set $\widetilde{\eta}_t=(1-e^{-m \eta_t})/m=\tau\cdot( \tau/\widetilde{\eta}_0+m t )^{-1}$ for some $\tau \geq 1$ and $\widetilde{\eta}_0 >0$. Then
\[
\sum_{t=1}^T
\widetilde{\eta}_t^{1-\tau} [Q(\widetilde{p}_t) - Q(p) ]
\leq \widetilde{\eta}_0^{-\tau} W_2^2(p_0,p) + 
G^2 \sum_{t=1}^T \widetilde{\eta}_t^{2-\tau} .
\]
\label{thm:sc-conv}
\end{proposition}
 
Choosing $\tau=2$ and $p=p_*$, we have
\begin{equation}
\sum_{t=1}^T
\frac{1+0.5 t}{T+0.25 T(T+1)}
\beta^{-1} [Q(\widetilde{p}_t) - Q(p_*) ]
\leq  \frac{4}{\beta m \widetilde{\eta}_0^2 T(T+1)} W_2^2(p_0,p_*) + \frac{4 G^2}{\beta m (T+1)} .
\label{eq:sc-tau=2}
\end{equation}
The learning rate schedule of $\eta_t = 1/m t$ (with $\tau=1$) was introduced to SGD analysis for strongly convex objectives in \cite{shalev2011pegasos}, which leads to a similar rate as that of Proposition~\ref{thm:sc-conv}, but with an extra $\log (T)$ term than \eqref{eq:sc-tau=2}. 
The use of $\tau>1$ has been adopted in more recent literature of SGD analysis, as an effort to avoid the $\log (T)$ term (for example, see \cite{LaScBa12}).
The resulting bound in the SGD analysis becomes identical to that of Proposition~\ref{thm:sc-conv}, and this rate is optimal for nonsmooth strongly convex optimization \cite{rakhlin12making}.
In addition, it is possible to implement for Langevin algorithm a similar scheme using moving averaging, as discussed in \cite{shamir2013stochastic}. 

It can be observed that taking a large step size $\widetilde{\eta}_0$ will grant rapid convergence.
The largest one can take is to choose $\eta_0=+\infty$ and consequently $\widetilde{\eta}_0=1/m$, leading to a learning rate schedule of $\widetilde{\eta}_t = 2/(m\cdot(t+2))$.
In this case, we are effectively initializing from $\widetilde{p}_1\propto\exp\lrp{-\beta^{-1} g(w)}$.
Choosing the same $p_0\propto\exp\lrp{-\beta^{-1} g(w)}$, we can bound the initial error $W_2^2(p_0,p_*)$ via the Talagrand inequality in Proposition~\ref{prop:KL} and the log-Sobolev inequality~\cite{bakryEmery,Ledoux_log_Sobolev_diffusion} for the $\beta^{-1}m$-strongly log-concave distribution $p_*$:
\[
W_2^2(p_0,p_*) \leq \frac{\beta}{m}\KL(p_*\|p_0) \leq \frac{\beta^2}{2m^2} \rE_{p_*}\lrb{\lrn{ \nabla \log\frac{p_*}{p_0}}^2} \leq \frac{G^2}{2m^2},
\]
since $\lrn{ \nabla \log\frac{p_*}{p_0}(w)} = \lrn{\beta^{-1}\nabla f(w)} \leq \beta^{-1}G$.
Plugging this bound and $\widetilde{\eta}_0=1/m$ into equation~\eqref{eq:sc-tau=2}, and noting that $T\geq 1$, we arrive at our result that
\begin{equation}
\sum_{t=1}^T
\frac{1+0.5 t}{T+0.25 T(T+1)}
\beta^{-1} [Q(\widetilde{p}_t) - Q(p_*) ]
\leq \frac{5 G^2}{\beta m T} .
\label{eq:sc-tau=2}
\end{equation}

\end{proof}

\begin{proof}[Proof of Proposition~\ref{thm:sc-conv}]
We can add the inequalities in Lemma~\ref{lem:entropy} and Lemma~\ref{lem:sgd0} to obtain:
\begin{align*}
\widetilde{\eta}_t [Q(\widetilde{p}_t) - Q(p) ]
\leq  e^{-m \eta_t} W_{2}({p}_{t-1},p)^2 - W_2(p_t,p)^2
+ \widetilde{\eta}_t^2 G^2.
\end{align*}
This is equivalent to
\begin{equation}
 \widetilde{\eta}_t^{1-\tau} [Q(\widetilde{p}_t) - Q(p) ]
\leq  (1-m \widetilde{\eta}_t)
\widetilde{\eta}_t^{-\tau}
W_{2}({p}_{t-1},p)^2 - \widetilde{\eta}_t^{-\tau} W_2(p_t,p)^2
+ \widetilde{\eta}_t^{2-\tau} G^2.
\label{eq:sc-onestep}
\end{equation}
We first show that 
\begin{equation}
(1-m \widetilde{\eta}_t)
\widetilde{\eta}_t^{-\tau}
\leq \widetilde{\eta}_{t-1}^{-\tau} .
\label{eq:lr-recursion}
\end{equation}
Let $s=t+\tau/(m\eta_0) \geq 1$ for $t\geq 1$, 
$\tilde{\eta}_t=\tau/(m s)$ and
$\tilde{\eta}_t=\tau/(m (s-1))$.
Therefore \eqref{eq:lr-recursion}
 is equivalent to
\[
(1-\tau/s) s^\tau \leq (s-1)^\tau .
\]
This inequality follows from the fact that for $z =1/s \in [0,1]$ and $\tau \geq 1$: 
$\psi(z)=(1-z)^\tau$ is convex in $z$, and thus $(1-\tau z)=\psi(0)+\psi'(0) z \leq \psi(z)= (1-z)^\tau$.

By combining \eqref{eq:sc-onestep} and \eqref{eq:lr-recursion}, we obtain
\begin{align*}
 \widetilde{\eta}_t^{1-\tau} [Q(\widetilde{p}_t) - Q(p) ]
\leq   
\widetilde{\eta}_{t-1}^{-\tau}
W_{2}({p}_{t-1},p)^2 - \widetilde{\eta}_t^{-\tau} W_2(p_t,p)^2
+ \widetilde{\eta}_t^{2-\tau} G^2.
\end{align*}
By summing over $t=1$ to $t=T$, we obtain the bound.

\end{proof}

\subsection{Streaming SGLD Convergence in Lipschitz Convex Case}

To analyze the streaming stochastic gradient Langevin algorithm, we assume that function $f$ decomposes:
\[
f(w) = \frac{1}{n} \sum_{i=1}^n \ell(w,z_i)
= \rE_{z\sim D} [\ell(w,z)],
\]
where $D$ is the distribution over the data samples.
In this case, we modify Assumption~\ref{assumption:Lip} and assume that the individual log-likelihood satisfies the Lipschitz condition.
\paragraph{Assumptions on individual loss $\ell$}
\begin{enumerate}[label=\rm{$\rm{A2_L^{SG}}$}, leftmargin=2.5\parindent]
    \item Function $\ell$ is $G_\ell$-Lipschitz continuous on $\rR^d$: $\|\nabla \ell(w,z)\|_2 \leq G_\ell$, $\forall z\in\Omega$. \label{assumption:Lip_SG}
\end{enumerate}


In the case that $\ell(w,z)$ is Lipschitz, our main result for SGLD is the following counterpart of Theorem~\ref{thm:1}.
\begin{theorem}
\label{thm:2}
Assume that function $f$ satisfies the convex assumption~\ref{assumption:1_Lip} and the Lipschitz continuous assumption for the individual log-likelihood~\ref{assumption:Lip_SG}. 
Further assume that function $g(w)$ satisfies Assumption~\ref{assumption:g}.
Then for $\widetilde{p}_T$ following the streaming SGLD Algorithm~\ref{alg:sgld}, it satisfies (
for $\widetilde{\eta}_t = \lrp{1-e^{-m\eta_t}}/m = 2/\lrp{ m (t+2) }$):
\begin{equation*}
\sum_{t=1}^T
\frac{1+0.5 t}{T+0.25 T(T+1)}
\beta^{-1} [Q(\widetilde{p}_t) - Q(p_*) ]
\leq \frac{5 G_\ell^2}{\beta m T} .
\end{equation*}
leading to the convergence time of
\[
T = \frac{5 G_\ell^2}{\beta m \epsilon},
\]
for the averaged distribution to convergence to $\epsilon$ accuracy in the KL-divergence $\beta^{-1} Q$.
\end{theorem}
We devote the rest of this section to prove Theorem~\ref{thm:2}.

\begin{proof}[Proof of Theorem~\ref{thm:2}]
Same as in the previous section, convergence of the regularized entropy $\rE_{w \sim p}\lrb{ g(w) } + H(p)$ along equation~\eqref{eq:sg_diffuse} follows Lemma~\ref{lem:entropy}.

For the convergence of the cross entropy $\rE_{w\sim p}\lrb{f(w)}$ along equation~\eqref{eq:sgld-2}, the following Lemma follows the standard analysis of SGD. 
\begin{lemma}
Adopt Assumption~\ref{assumption:Lip_SG} that $\ell(w,z)$ is $G_\ell$-Lipschitz for all $z\in\Omega$.
Also adopt Assumption~\ref{assumption:1_Lip} that $f(w)=\rE_{z\sim D} \ell(w,z)$ is convex.
We have for all $w \in \rR^d$:
\begin{equation}
2 \widetilde{\eta}_t \rE_{z\sim D}[ \ell(\widetilde{w}_t,z)- \ell(w,z)]
\leq \|\widetilde{w}_t -w \|_2^2 - \rE_{w_t|\widetilde{w}_t} \|w_{t}-w\|_2^2 + \widetilde{\eta}_t^2 G_\ell^2 .\label{eq:convex_result}
\end{equation}
\label{lem:convex_result}
\end{lemma}

It implies the following bound, which modifies Lemma~\ref{lem:sgd0}.
\begin{lemma}
Given any probability density $q$ on $\rR^d$. Define
\[
\ell(q)= \rE_{w \sim q} \rE_{z\sim D} \ell(w,z)  ,
\]
then we have
\[
2 \widetilde{\eta}_t [\ell(\widetilde{p}_t)- \ell(p)]
\leq  W_2(\widetilde{p}_t,p)^2 - W_2(p_t,p)^2
+ \widetilde{\eta}_t^2 G_\ell^2 . 
\]
\label{lem:sgd}
\end{lemma}

Initializing from the prior distribution, we can follow the same proof as in Proposition~\ref{thm:sc-conv} and obtain a similar convergence rate as in the non-stochastic case.


\begin{proposition}
Set $\widetilde{\eta}_t=(1-e^{-m \eta_t})/ m = \tau\cdot(\tau/\widetilde{\eta}_0+ m t)^{-1}$ for some $\tau \geq 1$ and $\widetilde{\eta}_0 >0$. Then
\[
\sum_{t=1}^T
\widetilde{\eta}_t^{1-\tau} [Q(\widetilde{p}_t) - Q(p) ]
\leq \eta_0^{-\tau} W_2(p_0,p)^2 + 
G_\ell^2 \sum_{t=1}^T \widetilde{\eta}_t^{2-\tau} .
\]
\label{thm:sg-conv}
\end{proposition}

We can choose $\tau=2$, and then for $p=p_*$, we have
\begin{equation}
\sum_{t=1}^T
\frac{1+0.5 t}{T+0.25 T(T+1)}
\beta^{-1} [Q(\widetilde{p}_t) - Q(p_*) ]
\leq  \frac{4}{\beta m \widetilde{\eta}_0^2 T(T+1)} W_2(p_0,p_*)^2 + \frac{4 G_\ell^2}{\beta m (T+1)} .
\label{eq:sg-tau=2}
\end{equation}
Following the same steps as in the full gradient case, we arrive at the result.
\end{proof}

\section{Langevin Algorithms in Smooth Convex Case}
\label{sec:smooth}
For the posterior $p(w|\vz) \propto \lrp{-\beta^{-1}(f(w)+g(w))}$, we make the following assumptions on function $f$.
\paragraph{Assumptions for the smooth convex case:}
\begin{enumerate}[label=\rm{A}{{\arabic*}}]
    \item Function $f:\rR^d\rightarrow\rR$ is convex and positive. \label{assumption:1}
\end{enumerate}
\begin{enumerate}[label=\rm{$\rm{A2_S}$}, leftmargin=2.0\parindent]
    \item Function $f$ is $L$-Smooth on $\rR^d$: $\|\nabla f(w)-\nabla f(w')\|_2 \leq L \|w-w'\|_2$. \label{assumption:2}
\end{enumerate}
We also assume that function $g:\rR^d\rightarrow\rR$ is $m$-strongly convex.
Note that this is equivalent to the cases where we simply assume the entire negative log-posterior to be $\beta^{-1} m$-strongly convex and $\lrp{ \beta^{-1} (L+m) }$-Lipschitz smooth: one can separate the negative log-posterior into two parts: $\frac{\beta^{-1}m}{2} \lrn{w}^2$ and $\lrp{ -\log p(w|\vz) - \frac{\beta^{-1}m}{2} \lrn{w}^2 }$, which is convex and $\beta^{-1} L$-Lipschitz smooth.
We therefore directly let $g(w) = \frac{m}{2} \lrn{w}^2$ in what follows.

%
\subsection{Full Gradient Langevin Algorithm Convergence in Smooth Convex Case}

Our main result for Full Gradient Langevin Algorithm in the case that $f$ is smooth can be stated as follows.
Compared to Theorem~\ref{thm:1}, the result of Theorem~\ref{thm:main_cvg} is useful for  loss functions such as least squares loss that are smooth but \emph{not Lipschitz continuous}.
\begin{theorem}
\label{thm:main_cvg}
Assume that function $f$ satisfies the convex and Lipschitz smooth Assumptions~\ref{assumption:1} and~\ref{assumption:2}. 
Also assume that $\nabla^2 f(w)\preceq H$.
Further let function $g(w) = \frac{m}{2} \lrn{w}^2$.
Then for $\widetilde{p}_T$ following Algorithm~\ref{alg:langevin} and initializing from $p_0\propto\exp(-\beta^{-1}g)$, it satisfies (
for $\widetilde{\eta}_t = \lrp{1-e^{-m\eta_t}}/m = 2\cdot\lrp{ 8 L + m t }^{-1}$):
\begin{multline*}
\sum_{t=1}^T
\frac{ 1/\lrp{m\widetilde{\eta}_0} + t/2 }{ T/\lrp{m\widetilde{\eta}_0} + T(T+1)/4 }
\beta^{-1} [Q(\widetilde{p}_t) - Q(p_*) ] \\
\leq \frac{64 L^2}{m^2 T(T+1)} \cdot \lrp{ \frac{8}{m^2} \trace\lrp{H^2} + 2 U(0) }
+ \frac{16}{T+1} \cdot \lrp{ \frac{8}{m^2} \trace\lrp{H^2} + 2 U(0) } .
\end{multline*}
leading to the convergence time of
\[
T = 64 \cdot \max\lrbb{ \frac{8\trace\lrp{H^2}}{m^2 \epsilon}, \frac{2 U(0) }{\epsilon} },
\]
for the averaged distribution to convergence to $\epsilon\leq1$ accuracy in the KL-divergence $\beta^{-1} Q$.
\end{theorem}

\paragraph{Ridge Separable Case}
%
Assume that function $f$ decomposes into the following ridge-separable form:
\begin{align}
f(w) = \frac{1}{n} \sum_{i=1}^n s_i(w^\top z_i),
\label{assumption:ridge}
\end{align}
We make some assumptions on the activation function $s_i$ and the data points $z_i$.
\label{sec:individual_assumption}
\paragraph{Assumptions in ridge separable case}
\begin{enumerate}[label=\rm{R}{{\arabic*}}]
    \item $\forall i\in\{1,\dots,n\}$, the one dimensional activation function $s_i(\cdot)$ has a bounded second derivative: $\left| s_i''(x) \right| \leq L_s$, for any $x\in\rR$. \label{assumption:ridge1}
    \item $\forall i\in\{1,\dots,n\}$, data point $z_i$ has a bounded norm: $\lrn{z_i}^2 \leq R_z$. \label{assumption:ridge2}
\end{enumerate}
Assumptions~\ref{assumption:ridge1} and~\ref{assumption:ridge2} combines to give a Lipschitz smoothness constant of $L_s R_z$ for the individual log-likelihood.
%
\begin{corollary}
We make the convexity Assumption~\ref{assumption:1} on function $f$ and let it take the ridge-separable form~\eqref{assumption:ridge} (also let function $g(w) = \frac{m}{2} \lrn{w}^2$). Further adopt Assumptions~\ref{assumption:ridge1} and~\ref{assumption:ridge2} on the activation functions and the data points, respectively.
Then the convergence time of Algorithm~\ref{alg:langevin} initializing from $p_0\propto\exp(-\beta^{-1}g)$ (with step size $\widetilde{\eta}_t = \lrp{1-e^{-m\eta_t}}/m = 2\lrp{ 8 L_s R_z + m t }^{-1}$) is
\[
T = 64 \cdot \max\lrbb{ \frac{ 8 L_s^2 R_z^2}{m^2 \epsilon}, \frac{2 U(0)}{\epsilon} },
\]
for the averaged distribution to convergence to $\epsilon$ accuracy in the KL-divergence $\beta^{-1} Q$.
\label{corollary:GD}
\end{corollary}
\begin{proof}
From Assumptions~\ref{assumption:ridge1} and~\ref{assumption:ridge2}, we know that $\nabla^2 f(w)\preceq \frac{1}{n} L_s Z Z^\top = H$,
where 
\[
\trace\lrp{H^2} = L_s^2 \frac{1}{n^2} \trace\lrp{Z^\top Z Z^\top Z}
= L_s^2 \frac{1}{n^2} \lrn{Z^\top Z}_F^2
= L_s^2 \frac{1}{n^2} \sum_{i,j=1}^n \lrp{z_i^\top z_j}^2
\leq L_s^2 \frac{1}{n} \sum_{i=1}^n \lrn{z_i}^4 
\leq L_s^2 R_z^2.
\]
Plugging this into Theorem~\ref{thm:main_cvg} leads to the convergence time of
\[
T = 64 \cdot \max\lrbb{ \frac{ 8 L_s^2 R_z^2}{m^2 \epsilon}, \frac{2 U(0)}{\epsilon} }.
\]
\end{proof}

We devote the rest of this section to the proof of Theorem~\ref{thm:main_cvg}.
\begin{proof}[Proof of Theorem~\ref{thm:main_cvg}]
Same as in Section~\ref{sec:ld_cvg_Lip}, convergence of the regularized entropy $\rE_{w \sim p}\lrb{ g(w) } + H(p)$ along equation~\eqref{eq:sg_diffuse} follows Lemma~\ref{lem:entropy}.

For the decrease of the cross entropy $\rE_{w\sim p}\lrb{f(w)}$ along the gradient descent step~\eqref{eq:ld}, we use the following derivation for $L$-Lipschitz smooth $f$. 
For $p_t$ being the density of $w_t$ following equation~\eqref{eq:ld} and for $p$ being another probability density,
\begin{align*}
&2 \widetilde{\eta}_t [
\rE_{w \sim \widetilde{p}_t} f(w) 
- \rE_{w' \sim p} f(w')]\\
\leq &
[W_2(\widetilde{p}_{t},p)^2
- W_{2}({p}_{t},p)^2]
+ \eta_t^2 \rE_{w \sim \widetilde{p}_t}
\|\nabla f(w)\|_2^2\\
\leq &
[W_2(\widetilde{p}_{t},p)^2
- W_{2}({p}_{t},p)^2]
+ 2 \widetilde{\eta}_t^2 \rE_{(w,w') \sim \gamma_t}
\|\nabla f(w)-\nabla f(w')\|_2^2
+ 2 \widetilde{\eta}_t^2 \rE_{w'\sim p} \|\nabla f(w')\|_2^2 ,
\end{align*}
where $\gamma_t \in \Gamma_{\text{opt}}(\widetilde{p}_t,p)$ is the optimal coupling between distributions with densities $\widetilde{p}_t$ and $p$.

With $\widetilde{\eta}_t=(1-e^{-m\eta_t})/m$, we have
\begin{align}
2\widetilde{\eta}_t
[Q(\widetilde{p}_t)-Q(p)]
&\leq (1-m \widetilde{\eta}_t)
W_2(\widetilde{p}_{t-1},p)^2
- W_{2}({p}_{t},p)^2\nonumber \\
& + 2 \widetilde{\eta}_t^2 \rE_{(w,w') \sim \gamma_t}
\|\nabla f(w)-\nabla f(w')\|_2^2+ 2 \widetilde{\eta}_t^2 \rE_{w'\sim p} \|\nabla f(w')\|_2^2 .
\label{eq:one-step-regret}
\end{align}

We also have the following lemma.
\begin{lemma}
 Let $\gamma_t \in \Gamma_{\text{opt}}(\widetilde{p}_t,p)$ be the optimal coupling of $\widetilde{p}_t$ and $p$, and let $p$ the solution of \eqref{eq:opt_ensemble}.
Then we have
\[
\rE_{(w,w') \sim \gamma_t}
\|\nabla f(w)-\nabla f(w')\|_2^2
\leq 2L [Q(\widetilde{p}_t)-Q(p)]  .
\]
\label{lem:grad-square-Q-loss}
\end{lemma}
We note that
Lemma~\ref{lem:grad-square-Q-loss} 
and equation~\eqref{eq:one-step-regret} imply that
\begin{align}
2 \widetilde{\eta}_t(1-2 L \widetilde{\eta}_t) [
Q(\widetilde{p}_t)
- Q(p)]
\leq 
 (1-m \widetilde{\eta}_t)
W_2(p_{t-1},p)^2
-W_2(p_{t},p)^2
+ 
2 \widetilde{\eta}_t^2 \rE_{w \sim p} \|\nabla f(w)\|_2^2 ,
\label{eq:onestep-regret-smooth}
\end{align}
where $p$ satisfies \eqref{eq:expo}. 

Next we bound the last term of equation~\eqref{eq:onestep-regret-smooth} at $p_*$: $\rE_{w \sim p_*} \|\nabla f(w)\|_2^2$.  

\begin{lemma}
Assume that
\[
\nabla^2 f(w) \preceq H , \quad g(w) = \frac{m}{2} \|w\|_2^2 .
\]
Let 
\[
w_*=\arg\min_w \left[ f(w) + g(w) \right] ,
\]
 and $p_* \propto \exp(-\beta^{-1}( f(w) + g(w)))$ be the solution of \eqref{eq:expo}. Then
\[
\rE_{w \sim p_*} \|\nabla f(w)\|_2^2 \leq 16 \frac{\beta}{m} \trace\lrp{H^2} + 2 m^2 \|w_*\|_2^2 .
\]
\label{lem:smooth-grad-square-bound}
\end{lemma}
With these lemmas, we are ready to prove the convergence time of the Langevin algorithm~\ref{alg:langevin}.
 
%
We note that the shrinking step size scheduling of $\widetilde{\eta}_t=\tau\cdot\lrp{\frac{\tau}{\widetilde{\tau}_0}+mt}^{-1}$ satisfies:
\[
\frac{\widetilde{\eta}_t^\tau}{\widetilde{\eta}_{t-1}^\tau}\leq\lrp{1 - m \widetilde{\eta}_t}.
\]
Using this inequality and combining Lemma~\ref{lem:smooth-grad-square-bound} and equation~\eqref{eq:onestep-regret-smooth} at $p=p_*$, we obtain that
\begin{align*}
&2 \widetilde{\eta}_t^{1-\tau} \lrp{1-2L \widetilde{\eta}_t} [
Q(\widetilde{p}_t)
- Q(p_*)] \\
\leq & 
\widetilde{\eta}_{t-1}^{-\tau}
W_2(p_{t-1},p_*)^2
-\widetilde{\eta}_{t}^{-\tau} W_2(p_{t},p_*)^2
+ 
4 \widetilde{\eta}_t^{2-\tau} 
\left[ 8 \lrp{\beta/m} \trace\lrp{H^2} + m^2 \lrn{w_*}^2 \right].
\end{align*}
Summing over $t=1,\dots,T$,
\begin{align*}
&2 \sum_{t=1}^T \widetilde{\eta}_t^{1-\tau} \lrp{1-2L \widetilde{\eta}_t} [
Q(\widetilde{p}_t)
- Q(p_*)] \\
\leq & 
\widetilde{\eta}_{0}^{-\tau}
W_2(p_{0},p_*)^2
+ 
4 \left[ 8 \lrp{\beta/m} \trace\lrp{H^2} + m^2 \lrn{w_*}^2 \right] \sum_{t=1}^T \widetilde{\eta}_t^{2-\tau} 
.
\end{align*}
Denote $\Delta = 4 \left[ 8 \lrp{\beta/m} \trace\lrp{H^2} + m^2 \lrn{w_*}^2 \right] $ and take $\widetilde{\eta}_t = \tau \cdot \lrp{\frac{\tau}{\widetilde{\eta}_0} + m t}^{-1}$.
Then for $\widetilde{\eta}_t \leq \frac{1}{4} \frac{1}{L}$ and for $\tau=2$,
\begin{equation}
m \sum_{t=1}^T
\lrp{ 1/\lrp{m\widetilde{\eta}_0} + t/2 }
[Q(\widetilde{p}_t) - Q(p_*) ]
\leq  \frac{1}{\widetilde{\eta}_0^2} W_2(p_0,p_*)^2 + \Delta \cdot T, \nonumber
\end{equation}
or
\begin{equation}
\sum_{t=1}^T
\frac{ 1/\lrp{m\widetilde{\eta}_0} + t/2 }{ T/\lrp{m\widetilde{\eta}_0} + T(T+1)/4 }
[Q(\widetilde{p}_t) - Q(p_*) ]
\leq  \frac{4}{m \widetilde{\eta}_0^2 T(T+1)} W_2(p_0,p_*)^2 + \frac{4 \Delta}{ m (T+1) }.
\end{equation}
Inspired by the Lipschitz continuous case, we take $p_0(w) \propto\exp\lrp{-\beta^{-1} g(w) }$.
Then by the Talagrand and log-Sobolev inequalities, 
$$W_2(p_0,p_*)^2 \leq \frac{\beta}{m}\mathrm{KL}\lrp{p_*\|p_0}  \leq \frac{\beta^2}{2m^2} \rE_{p_*}\lrb{\lrn{ \nabla \log\frac{p_*}{p_0}}^2} = \frac{\beta^2}{2m^2} \rE_{p_*}\lrb{\lrn{\beta^{-1} \nabla f(w)}^2}.$$
Applying Lemma~\ref{lem:smooth-grad-square-bound} to the above inequality, we obtain that $W_2(p_0,p_*)^2 \leq \frac{1}{m^2} \lrp{8 \lrp{\beta/m} \trace\lrp{H^2} + m^2 \lrn{w_*}^2}$.
Then taking $\widetilde{\eta}_0 = \frac{1}{4} \frac{1}{L}$, we obtain that the weighted-averaged KL divergence is upper bounded:
\begin{multline}
\sum_{t=1}^T
\frac{ 1/\lrp{m\widetilde{\eta}_0} + t/2 }{ T/\lrp{m\widetilde{\eta}_0} + T(T+1)/4 }
\beta^{-1} [Q(\widetilde{p}_t) - Q(p_*) ] \\
\leq \frac{64 L^2}{m^2 T(T+1)} \cdot \lrp{ \frac{8}{m^2} \trace\lrp{H^2} + \beta^{-1} m \lrn{w_*}^2 }
+ \frac{16}{T+1} \cdot \lrp{ \frac{8}{m^2} \trace\lrp{H^2} + \beta^{-1} m \lrn{w_*}^2 } .
\end{multline}
Since $L^2\leq \trace\lrp{H^2}$, $\forall\epsilon\leq 1$, the weighted-averaged KL divergence $\sum_{t=1}^T
\frac{ 1/\lrp{m\widetilde{\eta}_0} + t/2 }{ T/\lrp{m\widetilde{\eta}_0} + T(T+1)/4 } \beta^{-1}
[Q(\widetilde{p}_t) - Q(p_*) ] \leq \epsilon$ if 
\[
T \geq 64 \cdot \max\lrbb{ \frac{8\trace\lrp{H^2}}{m^2 \epsilon}, \frac{\beta^{-1}m \lrn{w_*}^2}{\epsilon} }.
\]
Plugging in the bound that $m \lrn{w_*}^2 \leq 2 f(0) = 2 \beta U(0)$ gives the final result.


\end{proof}

\subsection{SGLD Convergence in Smooth Convex Case}

Similar to the Lipschitz continuous case, we assume that function $f$ decomposes:
\[
f(w) = \frac{1}{n} \sum_{i=1}^n \ell(w,z_i)
= \rE_{z\sim D} [\ell(w,z)],
\]
where $D$ is the distribution over the data samples.
Making the following assumption, which modifies Assumption~\ref{assumption:2}, that the individual log-likelihood satisfies the Lipschitz smooth condition yields the convergence rate for the SGLD method.
\paragraph{Assumptions on individual loss $\ell$}
\begin{enumerate}[label=\rm{$\rm{A{{\arabic*}}_S^{SG}}$}, leftmargin=2.5\parindent, start=2]
    \item Function $\ell$ is $L_\ell$-Lipschitz smooth on $\rR^d$: $\ell(y,z)\geq\ell(x,z)+\nabla \ell(x,z)^\top (y-x) + \frac{1}{2L_{\ell}} \lrn{\nabla \ell(y,z)-\nabla \ell(x,z)}^2$, $\forall z\in\Omega$. \label{assumption:smooth_SG}
    \item The stochastic gradient variance at the mode $w_*$ is bounded: 
    $\rE_{z\sim D} \lrb{ \lrn{ \nabla \ell(w_*,z) - \nabla f(w_*) }^2 } \leq b^2$. \label{assumption:general_SG}
\end{enumerate}
Assumption~\ref{assumption:smooth_SG} ensures that the stochastic estimates of $f$ are $L_\ell$ Lipschitz smooth.

Under the above assumptions, we obtain in what follows the convergence rate for the SGLD method with minibatch size $|\mathcal{S}|$. This result is the counterpart of its full gradient version in Theorem~\ref{thm:main_cvg}.
\begin{theorem}
\label{thm:sgld_smooth}
We make the convexity Assumptions~\ref{assumption:1} on function $f$ and the regularity Assumptions~\ref{assumption:smooth_SG} and~\ref{assumption:general_SG} on its components $\ell$.
Also assume that $\nabla^2 f(w)\preceq H$.
Let function $g(w) = \frac{m}{2} \lrn{w}^2$.
Then taking $\widetilde{\eta}_t = \lrp{1-e^{-m\eta_t}}/m = 2\cdot\lrp{ 8 L_s R_z + m t }^{-1}$, the convergence time of the SGLD Algorithm~\ref{alg:sgld} initializing from $p_0\propto\exp(-\beta^{-1}g)$ is
\begin{align*}
T = \Omega\lrp{ \max\lrbb{ \frac{L_\ell \trace(H)}{m^2 \epsilon}, 
\frac{U(0)}{\epsilon}, 
\frac{n}{|\mathcal{S}|} \frac{b^2}{m \epsilon} } },
\end{align*}
to achieve an accuracy of 
\[\sum_{t=1}^T
\frac{ 1/\lrp{m\widetilde{\eta}_0} + t/2 }{ T/\lrp{m\widetilde{\eta}_0} + T(T+1)/4 } \beta^{-1}
[Q(\widetilde{p}_t) - Q(p_*) ] \leq \epsilon.
\]
\end{theorem}

\paragraph{Ridge Separable Case}
Assume that the individual component $\ell$ take the following form so that function $f$ becomes ridge-separable:
\begin{align}
\ell(w,z_i) = s_i(w^\top z_i).
\label{eq:ridge_sg}
\end{align}
To ensure bounded stochastic gradient variance at the mode of the posterior, we additionally assume that at the mode $w_*$, the derivatives of the activation functions are bounded.
\paragraph{Assumption in ridge separable case on bounded variance}
\begin{enumerate}[label=$\rm{R3^{SG}}$, leftmargin=2.5\parindent]
    \item $\exists b_s>0$, so that $\left|s_i'(w_*^\top z_i)\right| \leq b_s$, $\forall i\in\{1,\dots,n\}$, where $w_* = \arg\min_{w}\lrb{f(w)+g(w)}$.
    \label{assumption:ridge3}
\end{enumerate}
Assumption~\ref{assumption:ridge3} ensures that the stochastic gradient variance is bounded at the mode.
Then we have the following corollary instantiating Theorem~\ref{thm:sgld_smooth}.
\begin{corollary}
\label{corollary:sgld_smooth}
We make the convexity Assumption~\ref{assumption:1} on function $f$ and let it take the ridge-separable form~\eqref{assumption:ridge} (also let function $g(w) = \frac{m}{2} \lrn{w}^2$). Further adopt Assumptions~\ref{assumption:ridge1}, \ref{assumption:ridge2}, and~\ref{assumption:ridge3}.
Then taking $\widetilde{\eta}_t = \lrp{1-e^{-m\eta_t}}/m = 2\cdot\lrp{ 8 L_s R_z + m t }^{-1}$, the convergence time of Algorithm~\ref{alg:sgld} initializing from $p_0\propto\exp(-\beta^{-1}g)$ is
\[
T = \Omega\lrp{ \max\lrbb{ \frac{ L_s^2 R_z^2 }{m^2 \epsilon}, \frac{U(0)}{\epsilon}, \frac{n}{|\mathcal{S}|} \frac{ R_z b_s^2 }{m \epsilon} } },
\]
to achieve an accuracy of 
\[\sum_{t=1}^T
\frac{ 1/\lrp{m\widetilde{\eta}_0} + t/2 }{ T/\lrp{m\widetilde{\eta}_0} + T(T+1)/4 } \beta^{-1}
[Q(\widetilde{p}_t) - Q(p_*) ] \leq \epsilon.
\]
\end{corollary}
\begin{proof}[Proof of Corollary~\ref{corollary:sgld_smooth}]
Since function $\ell$ takes form~\eqref{eq:ridge_sg}, we can compute that $\nabla^2 \ell(w,z_i) = s_i(w^\top z_i) z_i z_i^\top$.
Using Assumptions~\ref{assumption:ridge1} and~\ref{assumption:ridge2}, the Lipschitz smoothness $L_\ell = L_s R_z$.
Same as in Corollary~\ref{corollary:GD}, we know that $\nabla^2 f(w)\preceq \frac{1}{n} L_s Z Z^\top = H$.
Therefore, 
$$\trace\lrp{H} = L_s \cdot \frac{1}{n} \trace\lrp{Z Z^\top} \leq L_s R_z,$$
leading to the fact that $ L_\ell\cdot\trace\lrp{H} \leq L_s^2 R_z^2$.

For the stochastic gradient bound $b$ at $w_*$, we apply Assumptions~\ref{assumption:ridge2} and~\ref{assumption:ridge3} to obtain
$$\lrn{ \nabla\ell(w_*,z_i) - \nabla\ell(w_*,z_j) } = \lrn{ s'_i(w_*^\top z_i)z_i-s'_j(w_*^\top z_j)z_j } \leq 2 \sqrt{R_z} b_s.$$
We thus have
\[
\lrn{ \nabla \ell(w_*,z) - \nabla f(w_*) }
= \Big\|\nabla \ell(w,z_i) - \frac{1}{n} \sum_{j=1}^n \nabla \ell(w,z_j)\Big\|
\leq 2 \sqrt{R_z} b_s,
\]
leading to the fact that 
\[
\rE_{z\sim D} \lrb{ \lrn{ \nabla \ell(w_*,z) - \nabla f(w_*) }^2 } \leq 4 R_z b_s^2.
\]
Therefore, the stochastic gradient variance bound in Assumption~\ref{assumption:general_SG}, $b = 2 \sqrt{R_z} b_s$.
Plugging these bounds into Theorem~\ref{thm:sgld_smooth} proves the corollary.
\end{proof}
We devote the rest of this section to the proof of Theorem~\ref{thm:sgld_smooth}.
\begin{proof}[Proof of Theorem~\ref{thm:sgld_smooth}]
We first note that because each $\ell(\cdot,z_i)$ is $L_\ell$-Lipschitz smooth, the stochastic estimate of function $f$, $\widetilde{f} = \frac{1}{|\mathcal{S}|} \sum_{z_i\in\mathcal{S}} \ell(\widetilde{w}_t,z_i)$ is $L_\ell$-Lipschitz smooth:
\begin{align*}
\lrn{ \frac{1}{|\mathcal{S}|} \sum_{z_i\in\mathcal{S}} \nabla\ell(y,z_i) - \nabla\ell(x,z_i) }^2
&\leq \frac{1}{|\mathcal{S}|^2} \lrp{ \sum_{z_i\in\mathcal{S}} \lrn{ \nabla\ell(y,z_i) - \nabla\ell(x,z_i) } }^2 \\
&\leq \frac{1}{|\mathcal{S}|} \sum_{z_i\in\mathcal{S}} \lrn{ \nabla\ell(y,z_i) - \nabla\ell(x,z_i) }^2 \\
&\leq \frac{2L_\ell}{|\mathcal{S}|} \sum_{z_i\in\mathcal{S}} \lrp{ \ell(y,z_i)-\ell(x,z_i)-\nabla \ell(x,z_i)^\top (y-x) } \\
&= 2 L_\ell \lrp{ \widetilde{f}(y) - \widetilde{f}(x) + \nabla \widetilde{f}(x)^\top (y-x) }.
\end{align*}
We thereby invoke the next lemma.
\begin{lemma}
Assume that function $f$ is convex, and that its stochastic estimate $\widetilde{f}$ is $L_\ell$-Lipschitz smooth.
Then
\begin{align*}
W_2^2(p_t, p)
&\leq W_2^2(\widetilde{p}_t, p) - 2 \widetilde{\eta}_t \lrp{ f(\widetilde{p}_t) - f(p) } \\
&+ \widetilde{\eta}_t^2 \lrp{ 4 L_\ell [Q(\widetilde{p}_t)-Q(p)] + 2 \rE_{ (\widetilde{w}_t,w') \sim \gamma_t} \lrb{ \rE_{\mathcal{S}|\widetilde{w}_t} \lrn{ \nabla \widetilde{f}(w',\mathcal{S}) }^2 } },
\end{align*}
where $f(q) = \rE_{w\sim q} f(w)$, and $\gamma_t \in \Gamma_{\mathrm{opt}}(\widetilde{p}_t,p)$ is the optimal coupling between $\widetilde{p}_t$ and $p$.
\label{lem:sgd_conv}
\end{lemma}
Taking $\widetilde{\eta}_t = \lrp{1-e^{-m\eta_t}}/m$ and combining Lemma~\ref{lem:entropy} and Lemma~\ref{lem:sgd_conv}, we obtain that
\begin{align}
2 \widetilde{\eta}_t \lrp{1-2\widetilde{\eta}_t L_\ell} \lrp{ Q(\widetilde{p}_t) - Q(p) } 
\leq e^{-m\eta_t} W_2^2(p_{t-1}, p) - W_2^2(p_t, p) 
+ 2 \widetilde{\eta}_t^2 \rE_{ (\widetilde{w}_t,w') \sim \gamma_t} \lrb{ \rE_{\mathcal{S}|\widetilde{w}_t} \lrn{ \nabla \widetilde{f}(w',\mathcal{S}) }^2 }.
\label{eq:sgd_conv}
\end{align}
We then adapt Lemma~\ref{lem:smooth-grad-square-bound} to the stochastic gradient method.
\begin{lemma}[Stochastic Gradient Counterpart of Lemma~\ref{lem:smooth-grad-square-bound}]
Assume that
\[
\nabla^2 f(w) \preceq H , \quad g(w) = \frac{m}{2} \|w\|_2^2 .
\]
Let 
\[
w_*=\arg\min_w \left[ f(w) + g(w) \right] ,
\]
 and $p$ be the solution of \eqref{eq:expo}. Then for $L_\ell$-Lipschitz smooth function $\widetilde{f}$, at $p=p_*$ and consequently $\gamma_t \in \Gamma_{\mathrm{opt}}(\widetilde{p}_t,p_*)$,
\begin{align}
\rE_{ (\widetilde{w}_t,w') \sim \gamma_t} \lrb{ \rE_{\mathcal{S}|\widetilde{w}_t} \lrn{ \nabla \widetilde{f}(w',\mathcal{S}) }^2 } 
\leq \frac{2 \beta L_\ell}{m} \trace(H)
+ 4 m^2 \lrn{w_*}^2 + 4 \rE_{\mathcal{S}} \lrn{ \nabla \widetilde{f}(w_*,\mathcal{S}) - \nabla f(w_*) }^2. 
\end{align}
\label{lem:smooth-SG-square-bound}
\end{lemma}

For the last piece of information, we establish the variance of the stochastic gradient at the mode, $\rE_{\mathcal{S}} \lrn{ \nabla \widetilde{f}(w_*,\mathcal{S}) - \nabla f(w_*) }_2^2$.
For samples $z_i$ that are i.i.d. draws from the data set and are unbiased estimators of $\nabla f(w_*) = \frac{1}{n} \sum_{j=1}^n \nabla \ell(w_*,z_j)$, we have
\begin{align*}
\rE_{\mathcal{S}} \lrb{\lrn{\sum_{z_i\in\mathcal{S}}\lrp{ \nabla \ell(w_*,z_i) - \frac{1}{n} \sum_{j=1}^n \nabla \ell(w_*,z_j) }}^2} 
= |\mathcal{S}| \cdot \rE_{z\sim D} \lrb{ \lrn{ \nabla \ell(w_*,z) - \frac{1}{n} \sum_{j=1}^n \nabla \ell(w_*,z_j) }^2 } 
\leq |\mathcal{S}| \cdot b^2,
\end{align*}
Leading to the bound that
\begin{align*}
\rE_{\mathcal{S}} \lrn{ \nabla \widetilde{f}(w_*,\mathcal{S}) - \nabla f(w_*) }_2^2
= \frac{1}{|\mathcal{S}|^2} \rE_{\mathcal{S}} \lrb{\lrn{\sum_{z_i\in\mathcal{S}}\lrp{ \nabla \ell(w_*,z_i) - \frac{1}{n} \sum_{j=1}^n \nabla \ell(w_*,z_j) }}^2} 
\leq \frac{b^2}{|\mathcal{S}|}.
\end{align*}

Plugging this result and Lemma~\ref{lem:smooth-SG-square-bound} into equation~\eqref{eq:sgd_conv} at $p=p_*$, we obtain the final bound that
\begin{align*}
2 \widetilde{\eta}_t \lrp{1-2\widetilde{\eta}_t L_s R_z} \lrp{ Q(\widetilde{p}_t) - Q(p_*) } 
&\leq \lrp{1-m\widetilde{\eta}_t} W_2^2(p_{t-1}, p_*) - W_2^2(p_t, p_*)  \\
&+ 4 \widetilde{\eta}_t^2 \lrp{ \beta \frac{L_\ell}{m} \trace(H)
+ 2 m^2 \lrn{w_*}^2 + 2 \frac{b^2}{|\mathcal{S}|} }.
\label{eq:sgd_conv}
\end{align*}
This leads to a convergence rate, similar to the full gradient case, of 
\begin{align*}
T = \Omega\lrp{ \max\lrp{ \frac{L_\ell \trace(H)}{m^2 \epsilon}, 
\frac{\beta^{-1} m \lrn{w_*}^2}{\epsilon}, 
\frac{\beta^{-1}}{|\mathcal{S}|} \frac{b^2}{m \epsilon} } },
\end{align*}
so that the weighted-averaged KL divergence:
\[
\sum_{t=1}^T
\frac{ 1/\lrp{m\widetilde{\eta}_0} + t/2 }{ T/\lrp{m\widetilde{\eta}_0} + T(T+1)/4 } \beta^{-1}
[Q(\widetilde{p}_t) - Q(p_*) ] \leq \epsilon.
\]
Since $\beta=1/n$, and that $m\lrn{w_*}^2 \leq 2 f(0) = 2 \beta U(0)$,
\begin{align*}
T = \Omega\lrp{ \max\lrp{ \frac{L_\ell \trace(H)}{m^2 \epsilon}, 
\frac{U(0)}{\epsilon}, 
\frac{n}{|\mathcal{S}|} \frac{b^2}{m \epsilon} } }.
\end{align*}
\end{proof}

\section{Proofs of the Supporting Lemmas}

\subsection{Proofs of Lemmas in the Lipschitz Continuous Case}
\subsubsection{Proof of Lemma~\ref{lem:entropy}}
\label{sec:relative_entropy}
Before proving Lemma~\ref{lem:entropy}, we first state a result in~\citep[Theorem 23.9 of][]{Villani} that establishes the strong subdifferential of the Wasserstein-2 distance.
\begin{lemma}
\label{lem_aux:subdif}
Assume that $\mu_t, \hat{\mu}_t$ solve the following continuity equations
\[
\frac{\partial \mu_t}{\partial t} + \nabla\cdot(\xi_t \mu_t) = 0, \qquad
\frac{\partial \hat{\mu}_t}{\partial t} + \nabla\cdot(\hat{\xi}_t \hat{\mu}_t) = 0.
\]
Then 
\[
\frac{1}{2} \frac{d}{d t} W_2^2( \mu_t, \hat{\mu}_t ) 
= - \int \lrw{ \tilde{\nabla} \psi_t, \xi_t } d \mu_t - \int \lrw{ \tilde{\nabla} \hat{\psi}_t, \hat{\xi}_t } d \hat{\mu}_t,
\]
where $\psi_t$ and $\hat{\psi}_t$ are the optimal transport vector fields:
\[
\exp(\tilde{\nabla} \psi_t)_{\#} \mu_t = \hat{\mu_t}, \qquad
\exp(\tilde{\nabla} \hat{\psi}_t)_{\#} \hat{\mu}_t = \mu_t.
\]
\end{lemma}
Writing $p_t$ and $\hat{p}_t$ as the density functions of $\mu_t$ and $\hat{\mu}_t$, we take $\xi_t = - \beta \nabla \log p_t - \nabla g$ and $\hat{\xi}_t = 0$ so that $\mu_t$ follows the Fokker-Planck equation  equation associated with process~\eqref{eq:diffuse} and $\hat{\mu}_t = \nu$ is a constant measure.
This leads to the following equation
\[
\frac{1}{2} \frac{d}{d s} W_2^2(\mu_s,\nu) 
= \int \lrw{\beta \nabla \log p_s + \nabla g, (\tilde{\nabla} \psi)_{\mu_s}^{\nu} } d \mu_s.
\]

For $\mu$ being the probability measure associated with its density $p$, define relative entropy $\beta^{-1} F(\mu)$, where $F(\mu) = \Ep{w\sim p}{g(w)} + H(p)$.
We can then use the fact that the relative entropy $\beta^{-1} F$ is $\beta^{-1} m$-geodesically strongly convex (see Proposition 9.3.2 of~\cite{gradient_flow}) to prove the following Lemma.
\begin{lemma}
\label{lem_aux:entropy}
For $p$ being the density of $\mu$,
\[
F(\nu) - F(\mu) - \frac{m}{2} W_2^2(\mu, \nu) \geq \int \lrw{\beta \nabla \log p + \nabla g, (\tilde{\nabla} \psi)_{\mu}^{\nu} } d \mu.
\]
\end{lemma}
\begin{proof}
Let $\mu_t$ be the geodesic between $\mu$ and $\nu$. $\beta^{-1} m$-geodesic strong convexity of $\beta^{-1} F$ states that (see Proposition 9.3.2 of~\cite{gradient_flow}):
\[
\beta^{-1} F(\mu_t) \leq t \beta^{-1} F(\nu) + (1-t) \beta^{-1} F(\mu) - \frac{\beta^{-1} m}{2} t (1-t) W_2^2 (\mu, \nu),
\]
and consequently
\begin{align*}
\frac{F(\mu_t) - F(\mu)}{t} \leq F(\nu) - F(\mu) - \frac{m}{2} (1-t) W_2^2(\mu, \nu).
\end{align*}
By the definition of subdifferential~\citep[c.f.][Theorem 23.14]{Villani} we also have along the diffusion process defined by equation~\eqref{eq:diffuse}:
\[
\lim\inf_{t\downarrow0} \frac{F(\mu_t) - F(\mu)}{t} 
\geq \int \lrw{ \beta \nabla\log p + \nabla g, (\tilde{\nabla} \psi)_{\mu}^{\nu} } d\mu.
\]
Taking the limit of $t\rightarrow0$, we obtain the result.
\end{proof}

\begin{proof}[Proof of Lemma~\ref{lem:entropy}]
Combining Lemma~\ref{lem_aux:subdif} and~\ref{lem_aux:entropy}, we obtain that 
\begin{align}
\frac{1}{2} \frac{d}{d s} W_2^2(\mu_s,\nu) 
= \int \lrw{\beta \nabla \log p_s + \nabla g, (\tilde{\nabla} \psi)_{\mu_s}^{\nu} } d \mu_s
\leq F(\nu) - F(\mu_s) - \frac{m}{2} W_2^2(\mu_s, \nu).
\label{eq:Wasserstein_entropy_diff}
\end{align}
Along the Fokker-Planck equation associated with process~\eqref{eq:diffuse}, $\frac{d}{d s} F(\mu_s) = - \Ep{p_s}{ \lrn{\beta \nabla \log p_s + \nabla g}_2^2 } \leq 0$, meaning that $F(\mu_s)$ is monotonically decreasing.
We obtain from equation~\eqref{eq:Wasserstein_entropy_diff} for $s\in[0,t]$,
\begin{align*}
\frac{1}{2} \frac{d}{d s} W_2^2(\mu_s,\nu) 
\leq \sup_{s\in[0,t]} \lrb{ F(\nu) - F(\mu_s) } - \frac{m}{2} W_2^2(\mu_s, \nu)
= F(\nu) - F(\mu_t) - \frac{m}{2} W_2^2(\mu_s, \nu).
\end{align*}
Applying the Gronwall's inequality, we arrive at the conclusion that
\begin{align*}
\frac{2}{m} \lrp{1- e^{-m \Delta t}} \lrp{F(\mu_t) - F(\nu)}
\leq e^{-m \Delta t} W_2^2(\mu_0,\nu) - W_2^2(\mu_t, \nu).
\end{align*}
Taking $d \mu_t = \tilde{p}_t dx $, $d \mu_0 = p_{t-1} d x$, $d \nu = p dx$, and ${\Delta t} = \eta_t$ finishes the proof.
\end{proof}

\subsubsection{Proof of Lemma~\ref{lem:sgd0}}
\begin{proof}[Proof of Lemma~\ref{lem:sgd0}]
We first state a point-wise result along the gradient descent step~\eqref{eq:ld}:
\begin{align}
2\eta_t \lrp{ f(\widetilde{w}_{t}) - f(w) } \leq \lrn{\widetilde{w}_{t}-w}_2^2 -  \lrn{w_{t}-w}_2^2 + \eta_t^2 G^2.
\label{eq:convex_result0}
\end{align}
This is because 
\begin{align*}
\lrn{ w_t - w }_2^2 
&= \lrn{ \tilde{w}_t - \eta_t \nabla f(\tilde{w}_t) - w }_2^2 \\
&= \lrn{ \tilde{w}_t - w }_2^2 - 2 \eta_t \lrw{ \nabla f(\tilde{w}_t) , \tilde{w}_t - w} + \eta_t^2 \lrn{ \nabla f(\tilde{w}_t) }_2^2 \\
&\leq \lrn{ \tilde{w}_t - w }_2^2 - 2 \eta_t \lrp{ f(\tilde{w}_t) - f(w) } + \eta_t^2 G^2,
\end{align*}
where the last step follows from the convexity and Lipschitz smoothness of $f$.

We then denote the measures corresponding to random variables $w_t$ and $\tilde{w}_t$ to be: $w_t\sim\mu_t$ and $\tilde{w}_t\sim\tilde{\mu}_t$. 
From the definitions, we know that they have densities $p_t$ and $\tilde{p}_t$.

Denote an optimal coupling between $\tilde{\mu}_t$ and $\mu$ (where measure $\mu$ has density $p$, which is the stationary distribution) to be $\gamma\in\Gamma_{opt}(\tilde{\mu}_t, \mu)$.
We then take expectations over $\gamma(\tilde{w}_t, w)$ on both sides of equation~\eqref{eq:convex_result0}:
\begin{align*}
2 \eta_t \lrp{ f(\tilde{p}_t)- f(p) }
&= 2 \eta_t \Ep{(\tilde{w}_t,w)\sim\gamma}{ f(\tilde{w}_t)- f(w) } \\
&\leq \Ep{(\tilde{w}_t,w)\sim\gamma}{ \|\tilde{w}_t -w \|_2^2 } 
- \Ep{(\tilde{w}_t,w)\sim\gamma}{ \|w_{t}-w\|_2^2 } + \eta_t^2 G^2 \\
&= W_2^2(\tilde{p}_t, p) - \Ep{(\tilde{w}_t,w)\sim\gamma}{ \|w_{t}-w\|_2^2 } + \eta_t^2 G^2.
\end{align*}
From the relationship $w_t = \tilde{w}_t - \eta_t \nabla f(\tilde{w}_t)$, we know that the joint distribution of $(w_t, w)$ is $\lrp{\mathrm{id} - \eta_t \nabla f, \, \mathrm{id}}_{\#} \gamma$.
Note that $\tilde{\gamma} = \lrp{\mathrm{id} - \eta_t \nabla f, \, \mathrm{id}}_{\#} \gamma$ also defines a coupling, and therefore
\begin{align*}
\Ep{(\tilde{w}_t,w)\sim\gamma}{ \|w_{t}-w\|_2^2 }
&= \Ep{(w_t,w)\sim\tilde{\gamma}}{ \|w_{t}-w\|_2^2 } \\
&\geq \inf_{\hat{\gamma}\in\Gamma(\mu_t,\mu)} \Ep{(w_t,w)\sim\hat{\gamma}}{ \|w_{t}-w\|_2^2 }
= W_2^2(p_t,p).
\end{align*}
Therefore, 
\begin{align*}
2 \eta_t \lrp{ f(\tilde{p}_t)- f(p) }
\leq W_2^2(\tilde{p}_t, p) - W_2^2(p_t,p) + \eta_t^2 G^2.
\end{align*}
\end{proof}

\subsubsection{Proofs of Lemma~\ref{lem:convex_result} and~\ref{lem:sgd} for the streaming SGLD algorithm~\ref{alg:sgld}}
\begin{proof}[Proof of Lemma~\ref{lem:convex_result}]
By the definitions of $w_t$ and $\tilde{w}_t$,
\begin{align*}
\lrn{ w_t - w }_2^2 
&= \lrn{ \tilde{w}_t - \eta_t \nabla\ell(\tilde{w}_t,z_t) - w }_2^2 \\
&= \lrn{ \tilde{w}_t - w }_2^2 - 2 \eta_t \lrw{ \nabla\ell(\tilde{w}_t,z_t) , \tilde{w}_t - w} + \eta_t^2 \lrn{ \nabla \ell(\tilde{w}_t,z_t) }_2^2.
\end{align*}

We now take expectation with respect to $z_t$, conditioned on $\tilde{w}_t$, to obtain
\begin{align*}
\rE_{z_t|\tilde{w}_t}
\lrn{ w_t - w }_2^2 
&\leq \lrn{ \tilde{w}_t - w }_2^2 - 2 \eta_t \lrw{ \nabla f(\tilde{w}_t) , \tilde{w}_t - w} + \eta_t^2  G_\ell^2 \\
&\leq \lrn{ \tilde{w}_t - w }_2^2 - 2 \eta_t \lrp{ f(\tilde{w}_t) - f(w) } + \eta_t^2 G_\ell^2.
\end{align*}
The last step follows from the convexity of $f$.
Therefore, the desired bound follows.
\end{proof}

\begin{proof}[Proof of Lemma~\ref{lem:sgd}]
We first denote the measures corresponding to random variables $w_t$ and $\tilde{w}_t$ to be: $w_t\sim\mu_t$ and $\tilde{w}_t\sim\tilde{\mu}_t$. 
From the definitions, we know that they have densities $p_t$ and $\tilde{p}_t$.

Denote an optimal coupling between $\tilde{\mu}_t$ and $\mu$ (where measure $\mu$ has density $p$, which is the stationary distribution) to be $\gamma\in\Gamma_{opt}(\tilde{\mu}_t, \mu)$.
We then take expectations over $\gamma(\tilde{w}_t, w)$ on both sides of Eq.~\eqref{eq:convex_result}, $\forall z\in\Omega$:
\begin{align*}
&2 \eta_t \Ep{(\tilde{w}_t,w)\sim\gamma}{ \Ep{z}{ \ell(\tilde{w}_t, z)- \ell(w, z) } } \\
&\leq \Ep{(\tilde{w}_t,w)\sim\gamma}{ \|\tilde{w}_t -w \|_2^2 } 
- \Ep{(\tilde{w}_t,w)\sim\gamma}{ \rE_{w_t}\lrb{ \|w_{t}-w\|_2^2 | \widetilde{w}_t } } + \eta_t^2 G_\ell^2 \\
&= W_2^2(\tilde{p}_t, p) - \Ep{(\tilde{w}_t,w)\sim\gamma}{ \rE_{w_t}\lrb{ \|w_{t}-w\|_2^2 | \widetilde{w}_t } } + \eta_t^2 G_\ell^2.
\end{align*}

From the relationship $w_t = \tilde{w}_t - \eta_t \nabla \ell(\tilde{w}_t, z_t)$, we know that conditional on $z_t$, the joint distribution of $(w_t, w)$ is $\lrp{\mathrm{id} - \eta_t \nabla\ell, \, \mathrm{id}}_{\#} \gamma$.
Note that $\tilde{\gamma} = \lrp{\mathrm{id} - \eta_t \nabla\ell, \, \mathrm{id}}_{\#} \gamma$ also defines a coupling, and therefore
\begin{align*}
&\Ep{(\tilde{w}_t,w)\sim\gamma}{ \rE_{w_t}\lrb{ \|w_{t}-w\|_2^2 | \widetilde{w}_t } } \\
&= \Ep{z_t}{ \Ep{(\tilde{w}_t,w)\sim\gamma}{ \rE_{w_t}\lrb{ \|w_{t}-w\|_2^2 | \widetilde{w}_t, z_t } } } \\
&= \Ep{z_t}{ \Ep{(w_t,w)\sim\tilde{\gamma}}{ \|w_{t}-w\|_2^2 | z_t } } \\
&\geq \inf_{\hat{\gamma}\in\Gamma(\mu_t,\mu)} \Ep{(w_t,w)\sim\hat{\gamma}}{ \|w_{t}-w\|_2^2 }
= W_2^2(p_t,p).
\end{align*}
Plugging this result and the Lipschitz assumption on $\ell$ in, 
we obtain that
\[
2 \widetilde{\eta}_t [\ell(\widetilde{p}_t)- \ell(p)]
\leq  W_2(\widetilde{p}_t,p)^2 - W_2(p_t,p)^2
+ \widetilde{\eta}_t^2 G_\ell^2 . 
\]
\end{proof}

\subsection{Proofs of Lemmas in the Lipschitz Smooth Case}
\subsubsection{Proofs of Lemmas~\ref{lem:grad-square-Q-loss} and~\ref{lem:smooth-grad-square-bound} for the full gradient Langevin algorithm~\ref{alg:langevin}}
\begin{proof}[Proof of Lemma~\ref{lem:grad-square-Q-loss}]
By the geodesic convexity of the entropy function $H(p) = \beta \rE_{w\sim p}\lrb{ \ln p(w) }$,
\[
H(\widetilde{p_t}) - H(p) \geq \beta \int \lrw{ \nabla \ln p(w'), \lrp{ T_{p'}^{p} - \mathrm{id} }(w') } p(w') \ d w'.
\]
where $T_{p}^{\widetilde{p}_t}$ is the optimal transport from $p$ to $\widetilde{p}_t$.
Using optimal coupling $\mu_t\in\Pi(\widetilde{p},p)$,
\[
H(\widetilde{p_t}) - H(p) \geq \beta \rE_{w,w'\sim\mu_t} \lrb{ \lrw{ \nabla \ln p(w'), w- w' }}.
\]
In addition, convexity of $f$ and $g$ implies that
\[
\rE_{w,w' \sim\mu_t} 
[g(w) - g(w')] \geq \rE_{w,w'\sim\mu_t} \lrb{ \lrw{ \nabla g(w'), w- w'}}
\]
and 
\[
\rE_{w,w' \sim\mu_t} 
[f(w) - f(w')] \geq \rE_{w,w'\sim\mu_t} \lrb{ \lrw{ \nabla f(w'), w- w' } }.
\]
Adding the above three inequalities, and note that the following holds point-wise
\[
\beta \nabla \ln p(w') + \nabla g(w') + \nabla f(w') = 0 ,
\]
we obtain that
\[
H(\widetilde{p_t}) - H(p) + \rE_{w,w' \sim\mu_t} [g(w) - g(w')] \geq - \rE_{w,w'\sim\mu_t} \lrb{ \lrw{ \nabla f(w'), w- w' }},
\]
and that
\begin{align}
Q(\widetilde{p_t}) - Q(p)
\geq 
\rE_{w,w' \sim\mu_t} 
[f(w) - f(w')] - 
\rE_{w,w'\sim\mu_t} \lrb{ \lrw{ \nabla f(w'), w- w' }}.     \label{eq:Q_diff_U_diff}
\end{align}

Since the potential function $f(w)$ is convex and $L$-smooth, 
\begin{equation}
f(w) \geq f(w') + \nabla f(w')^\top (w-w') + \frac{1}{2L} \|\nabla f(w')-\nabla f(w)\|_2^2 .
\label{eq:smooth}
\end{equation}
Combining equations~\eqref{eq:Q_diff_U_diff} and~\eqref{eq:smooth}, we obtain the desired bound. 
\end{proof}


%
\begin{proof}[Proof of Lemma~\ref{lem:smooth-grad-square-bound}]
We have
\begin{align*}
\rE_{w \sim p} \|\nabla f(w)\|_2^2 
&\leq
2  \rE_{w \sim p} \|\nabla f(w)-\nabla f(w_*)\|_2^2
+ 2 \|\nabla f(w_*)\|_2^2 \\
&\leq
2 \rE_{w \sim p}\lrn{ \lrp{\int_0^1 \nabla^2 f\lrp{\tau w + (1-\tau)w_*} d \tau}\lrp{w-w_*} }^2 +2 m^2 \|w_*\|_2^2 \\
&\leq 2 \rE_{w\sim p} \lrn{H (w-w_*)}^2 + 2 m^2 \|w_*\|_2^2.
\end{align*}
In what follows we bound $\rE_{w\sim p} \lrn{H (w-w_*)}^2$.
We first note that the posterior density:
\[
p \propto \exp(-\beta^{-1}( f(w) + g(w)))
\]
is $\frac{m}{\beta}$ strongly log-concave. 
By \cite{Harge_2004}, for any convex function $h$,
\[
\rE_{w\sim p} \lrb{ h\lrp{w-\rE_{w\sim p} \lrb{w}} }
\leq \rE_{\tilde{w}\sim\mathcal{N}\lrp{0,\beta/m \cdot \mI}} \lrb{ h(\tilde{w}) }.
\]
Taking $h(w) = \lrn{H w}^2$, we obtain that 
\[
\rE_{w\sim p} \lrn{ H\lrp{w-\rE_{w\sim p} \lrb{w}} }^2
\leq \rE_{\tilde{w}\sim\mathcal{N}\lrp{0,\beta/m \cdot \mI}} \lrn{ H \tilde{w} }^2
= \frac{\beta}{m} \trace\lrp{H^2}.
\]
On the other hand, by the celebrated relation between mean and mode for 1-unimodal distributions~\citep[see, e.g., Theorem $7$ of][]{Basu_96}, 
\[
\lrp{w_*-\rE_{w\sim p}[w]}^\top \Sigma^{-1} \lrp{w_*-\rE_{w\sim p}[w]} \leq 3,
\]
where $\Sigma$ is the covariance of $p$.
This results in the following bound
\[
\lrn{H\lrp{w_*-\rE_{w\sim p}[w]}}^2 \leq 3 \frac{\beta}{m} \lrn{H}_2^2.
\]
Combining the two bounds, we obtain that 
\begin{align*}
\rE_{w\sim p} \lrn{H (w-w_*)}^2
&\leq 2 \rE_{w\sim p} \lrn{ H\lrp{w-\rE_{w\sim p} \lrb{w}} }^2 + 2 \lrn{H\lrp{w_*-\rE_{w\sim p}[w]}}^2 \\
&\leq 2 \frac{\beta}{m} \trace\lrp{H^2} + 6 \frac{\beta}{m} \lrn{H}_2^2
\leq 8 \frac{\beta}{m} \trace\lrp{H^2}.
\end{align*}
\end{proof}

\subsubsection{Proofs of Lemma~\ref{lem:sgd_conv} and~\ref{lem:smooth-SG-square-bound} for the SGLD algorithm~\ref{alg:sgld}}
\begin{proof}[Proof of Lemma~\ref{lem:sgd_conv}]
By the definitions of $w_t$ and $\tilde{w}_t$,
\begin{align*}
\lrn{ w_t - w }_2^2 
&= \lrn{ \tilde{w}_t - \eta_t \nabla\widetilde{f}(\tilde{w}_t,\mathcal{S}) - w }_2^2 \\
&= \lrn{ \tilde{w}_t - w }_2^2 - 2 \eta_t \lrw{ \nabla\widetilde{f}(\tilde{w}_t,\mathcal{S}) , \tilde{w}_t - w} + \eta_t^2 \lrn{ \nabla \widetilde{f}(\tilde{w}_t,\mathcal{S}) }_2^2.
\end{align*}
We now take expectation with respect to $\mathcal{S}$, conditioned on $\tilde{w}_t$, to obtain
\begin{align}
\rE_{w_t|\tilde{w}_t}
\lrn{ w_t - w }_2^2 
&= \lrn{ \tilde{w}_t - w }_2^2 - 2 \eta_t \lrw{ \nabla f(\tilde{w}_t) , \tilde{w}_t - w} + \eta_t^2 \rE_{\mathcal{S}|\tilde{w}_t} \lrn{ \nabla \widetilde{f}(\tilde{w}_t,\mathcal{S}) }_2^2 \nonumber\\
&\leq \lrn{ \tilde{w}_t - w }_2^2 - 2 \eta_t \lrp{ f(\tilde{w}_t) - f(w) } + \eta_t^2 \rE_{\mathcal{S}|\tilde{w}_t} \lrn{ \nabla \widetilde{f}(\tilde{w}_t,\mathcal{S}) }_2^2.
\label{eq:w_dist_bound}
\end{align}

We then upper bound $\rE_{\mathcal{S}|\tilde{w}_t} \lrn{ \nabla \widetilde{f}(\tilde{w}_t,\mathcal{S}) }_2^2$ by introducing variable $w'$ that is distributed according to $p$ and couples optimally with the law of $\tilde{w}_t$:
\begin{align}
\rE_{\mathcal{S}|\tilde{w}_t} \lrn{ \nabla \widetilde{f}(\tilde{w}_t,\mathcal{S}) }_2^2
\leq 2 \rE_{\mathcal{S}|\tilde{w}_t} \lrn{ \nabla \widetilde{f}(\tilde{w}_t,\mathcal{S}) - \nabla \widetilde{f}(w',\mathcal{S}) }^2 + 2 \rE_{\mathcal{S}|\tilde{w}_t} \lrn{ \nabla \widetilde{f}(w',\mathcal{S}) }^2.
\label{eq:SG_square_bound}
\end{align}

For function $\widetilde{f}$ being $L_\ell$-Lipschitz smooth,
\begin{equation*}
\widetilde{f}(\tilde{w}_t,\mathcal{S}) \geq \widetilde{f}(w',\mathcal{S}) + \nabla \widetilde{f}(w',\mathcal{S})^\top (w-w') + \frac{1}{2 L_\ell} \|\nabla \widetilde{f}(w',\mathcal{S}) - \nabla \widetilde{f}(\tilde{w}_t,\mathcal{S}) \|_2^2.
\end{equation*}
Taking expectation over the randomness of minibatch assignment $\mathcal{S}$ on both sides leads to the fact that
\begin{equation*}
f(\tilde{w}_t) \geq f(w') + \nabla f(w')^\top (w-w') + \frac{1}{2 L_\ell} \rE_{\mathcal{S}|\tilde{w}_t} \| \nabla \widetilde{f}(w',\mathcal{S}) - \nabla \widetilde{f}(\tilde{w}_t,\mathcal{S}) \|_2^2 .
\end{equation*}
Combining this equation with equation~\eqref{eq:Q_diff_U_diff}, we adapt Lemma~\ref{lem:grad-square-Q-loss} to the stochastic gradient method:
\[
\rE_{ (\tilde{w}_t,w') \sim \mu_t} \lrb{ \rE_{\mathcal{S}|\tilde{w}_t} 
\| \nabla \widetilde{f}(w',\mathcal{S}) - \nabla \widetilde{f}(\tilde{w}_t,\mathcal{S}) \|_2^2 }
\leq 2 L_\ell [Q(\widetilde{p}_t)-Q(p)]  .
\]

Applying this result to equation~\eqref{eq:SG_square_bound} and taking expectation of $ (\tilde{w}_t,w') \sim \mu_t $ on both sides, we obtain:
\[
\rE_{\tilde{w}_t \sim \widetilde{p}_t} \lrb{ \rE_{\mathcal{S}|\tilde{w}_t} \lrn{ \nabla \widetilde{f}(\tilde{w}_t,\mathcal{S}) }_2^2 }
\leq 4 L_\ell [Q(\widetilde{p}_t)-Q(p)] + 2 \rE_{ (\tilde{w}_t,w') \sim \mu_t} \lrb{ \rE_{\mathcal{S}|\tilde{w}_t} \lrn{ \nabla \widetilde{f}(w',\mathcal{S}) }^2 }.
\]
Therefore, 
\begin{align*}
\rE_{ (\tilde{w}_t,w') \sim \mu_t}  \lrb{ \rE_{w_t|\tilde{w}_t}
\lrn{ w_t - w }_2^2 }
&\leq \rE_{ (\tilde{w}_t,w') \sim \mu_t} \lrb{ \lrn{ \tilde{w}_t - w }_2^2 } - 2 \eta_t \lrp{ f(\widetilde{p}_t) - f(p) } \\
&+ \eta_t^2 \lrp{ 4 L_\ell [Q(\widetilde{p}_t)-Q(p)] + 2 \rE_{ (\tilde{w}_t,w') \sim \mu_t} \lrb{ \rE_{\mathcal{S}|\tilde{w}_t} \lrn{ \nabla \widetilde{f}(w',\mathcal{S}) }^2 } },
\end{align*}
leading to the final result that
\begin{align*}
W_2^2(p_t, p)
&\leq \rE_{ (\tilde{w}_t,w') \sim \mu_t}  \lrb{ \rE_{w_t|\tilde{w}_t}
\lrn{ w_t - w }_2^2 } \\
&\leq W_2^2(\widetilde{p}_t, p) - 2 \eta_t \lrp{ f(\widetilde{p}_t) - f(p) } \\
&+ \eta_t^2 \lrp{ 4 L_\ell [Q(\widetilde{p}_t)-Q(p)] + 2 \rE_{ (\tilde{w}_t,w') \sim \mu_t} \lrb{ \rE_{\mathcal{S}|\tilde{w}_t} \lrn{ \nabla \widetilde{f}(w',\mathcal{S}) }^2 } }.
\end{align*}
\end{proof}

\begin{proof}[Proof of Lemma~\ref{lem:smooth-SG-square-bound}]
We have
\begin{align*}
\rE_{\mathcal{S}|\tilde{w}_t} \lrn{ \nabla \widetilde{f}(w',\mathcal{S}) }^2
\leq &
2  \rE_{\mathcal{S}|\tilde{w}_t} \| \nabla \widetilde{f}(w',\mathcal{S})-\nabla \widetilde{f}(w_*,\mathcal{S}) \|_2^2
+ 2 \rE_{\mathcal{S}|\tilde{w}_t} \| \nabla \widetilde{f}(w_*,\mathcal{S}) \|_2^2 \\
\leq & 
4 L_\ell \left[ f(w)- f(w_*) - \nabla f(w_*)^\top (w-w_*)\right] +2 \rE_{\mathcal{S}|\tilde{w}_t} \| \nabla \widetilde{f}(w_*,\mathcal{S}) \|_2^2.
\end{align*}
Taking expectation on both sides, we obtain that 
\begin{align}
\rE_{ (\tilde{w}_t,w') \sim \mu_t} \lrb{ \rE_{\mathcal{S}|\tilde{w}_t} \lrn{ \nabla \widetilde{f}(w',\mathcal{S}) }^2 } 
&\leq 4 L_\ell \rE_{w \sim p} \left[ f(w)- f(w_*) - \nabla f(w_*)^\top (w-w_*)\right] \nonumber\\
&+ 2 \rE_{\mathcal{S}} \lrn{ \nabla \widetilde{f}(w_*,\mathcal{S}) }^2. 
\label{eq:sg_bound}
\end{align}

We now upper bound $\rE_{w \sim p} \left[ f(w)- f(w_*) - \nabla f(w_*)^\top (w-w_*)\right]$.
Let $p_0$ be the normal distribution $N(w_*,(\beta/m) I)$, and define
\begin{align}
\Delta f(w)&= f(w)-f(w_*)-\nabla f(w_*)^\top (w-w_*) , \label{eq:Delta_U_def}\\
\Delta g(w) &= g(w) - g(w_*)- \nabla g(w_*)^\top (w-w_*) ,
\end{align}
then $p$ can be expressed as
\[
p \propto \exp(-\beta^{-1}(\Delta f(w) + \Delta g(w))) ,
\]
which is the solution of
\[
p = \arg\min_p \rE_{w \sim p} \left[ \Delta f(w) + \beta \ln \frac{p(w)}{p_0(w)}\right] .
\]
Therefore 
\begin{align}
\rE_{w \sim p} \Delta f(w) \leq&
\rE_{w \sim p} \left[ \Delta f(w) + \beta \ln \frac{p(w)}{p_0(w)}\right] \nonumber\\
=& -\beta
\ln \rE_{w \sim p_0} \exp(-\beta^{-1} \Delta f(w)) \nonumber\\
\leq& 
-\beta
\ln \rE_{w \sim p_0} \exp(- 0.5 \beta^{-1} (w-w_*)^\top H (w-w_*)) \nonumber\\
=&  0.5 \beta \ln |I + H/m| 
\leq 0.5 \beta \trace(H/m) .
\label{eq:Delta_U_bound}
\end{align}

We then decompose $\rE_{\mathcal{S}} \lrn{ \nabla \widetilde{f}(w_*,\mathcal{S}) }^2$:
\[
\rE_{\mathcal{S}} \lrn{ \nabla \widetilde{f}(w_*,\mathcal{S}) }^2
\leq 2 \lrn{ \nabla f(w_*) }^2 + 2 \rE_{\mathcal{S}} \lrn{ \nabla f(w_*) - \nabla \widetilde{f}(w_*,\mathcal{S}) }^2.
\]
Since $w_*$ is the minimum of $f+g$, $\nabla f(w_*) = - \nabla g(w_*) = - m w_*$.
Hence
\begin{align}
\rE_{\mathcal{S}} \lrn{ \nabla \widetilde{f}(w_*,\mathcal{S}) }^2
\leq 2 m^2 \lrn{w_*}^2 + 2 \rE_{\mathcal{S}} \lrn{ \nabla f(w_*) - \nabla \widetilde{f}(w_*,\mathcal{S}) }^2.
\label{eq:grad_bound}
\end{align}

Plugging inequalities~\eqref{eq:Delta_U_bound} and~\eqref{eq:grad_bound} into inequality~\eqref{eq:sg_bound} proves the desired result that
\begin{align}
\rE_{ (\tilde{w}_t,w') \sim \mu_t} \lrb{ \rE_{\mathcal{S}|\tilde{w}_t} \lrn{ \nabla \widetilde{f}(w',\mathcal{S}) }^2 } 
\leq \frac{2 \beta L_\ell}{m} \trace(H)
+ 4 m^2 \lrn{w_*}^2 + 4 \rE_{\mathcal{S}} \lrn{ \nabla f(w_*) - \nabla \widetilde{f}(w_*,\mathcal{S}) }^2. 
\end{align}
\end{proof}

\section{Conclusion}

This paper investigated the convergence of Langevin algorithms with strongly log-concave posteriors. We assume that the strongly log-concave posterior can be decomposed into two parts, with one part being simple and explicitly integrable with respect to the underlying SDE. This is analogous to the situation of proximal gradient methods in convex optimization. 
Using a new analysis technique which mimics the corresponding analysis of convex optimization, we obtain convergence results for Langenvin algorithms that are independent of dimension, both for Lipschitz and for a large class of smooth convex problems in machine learning. Our result addresses a long-standing puzzle with respect to the convergence of the Langevin  algorithms. 
We note that the current work focused on the standard Langevin algorithm, and the resulting convergence rate in terms of $\epsilon$ dependency is inferior to the best known results leveraging underdamped or even higher order Langevin dynamics such as~\cite{Xiang_underdamped,Dalalyan_underdamped,shen2019randomized,Yian_underdamped,MCMC_higher}, which corresponds to accelerated methods in optimization. It thus remains open to investigate whether dimension independent bounds can be combined with these accelerated methods to improve $\epsilon$ dependence as well as condition number dependence.





\end{document}